\documentclass{article}

\PassOptionsToPackage{numbers, compress, sort}{natbib}

\usepackage[final]{neurips_2021}

\usepackage[utf8]{inputenc} 
\usepackage[T1]{fontenc}    
\usepackage{url}            
\usepackage{booktabs}       
\usepackage{amsfonts}       
\usepackage{nicefrac}       
\usepackage{microtype}      
\usepackage{xcolor}         
\usepackage[disable]{todonotes}
\usepackage{amsthm}
\usepackage{amssymb}
\usepackage{mathtools}
\usepackage{xcolor}
\usepackage{mdframed}
\usepackage{subcaption}
\usepackage[frozencache=true,cachedir=minted-cache]{minted}

\usepackage[linesnumbered, ruled,vlined]{algorithm2e}
 
\usetikzlibrary{arrows,positioning,backgrounds} 

\usepackage[colorlinks=true,
linkcolor=purple,
citecolor=purple]{hyperref}

\newtheorem{definition}{Definition}

\newtheorem{theorem}{Theorem}
\newtheorem{proposition}{Proposition}
\newtheorem{corollary}{Corollary}
\newtheorem{lemma}{Lemma}
\newtheorem{claim}{Claim}

\newenvironment{thmbis}[1]
  {\renewcommand{\thetheorem}{\ref{#1}}%
   \addtocounter{theorem}{-1}%
   \begin{theorem}}
  {\end{theorem}}
  
\newenvironment{propbis}[1]
  {\renewcommand{\theproposition}{\ref{#1}}%
   \addtocounter{proposition}{-1}%
   \begin{proposition}}
  {\end{proposition}}

\usepackage{amsmath,amsfonts,bm}

\def\eqref#1{equation~\ref{#1}}

\def\1{\bm{1}}

\def\vzero{{\bm{0}}}
\def\vone{{\bm{1}}}

\def\vw{{\bm{w}}}
\def\vx{{\bm{x}}}

\DeclareMathAlphabet{\mathsfit}{\encodingdefault}{\sfdefault}{m}{sl}
\SetMathAlphabet{\mathsfit}{bold}{\encodingdefault}{\sfdefault}{bx}{n}

\newcommand\false{\mathbf{false}}
\newcommand\true{\mathbf{true}}

\newcommand{\np}{\text{\rm NP}}

\newcommand{\M}{\mathcal{M}}

\newcommand{\C}{\mathcal{C}}

\newcommand{\T}{\mathcal{T}}

\newcommand{\logicEvaluation}{\textsc{Eval}}
\newcommand{\perceptrons}{\mathsf{Ptron}}
\newcommand{\pos}{\textsc{Pos}}
\newcommand{\full}{\textsc{Full}}
\newcommand{\sr}{\textsc{SR}}
\newcommand{\msr}{\textsc{mSR}}
\newcommand{\astruct}{\mathfrak{A}}
\newcommand{\FO}{\text{\rm FO}} 
\newcommand{\es}{\mathbf{e}}

    \newcommand{\foil}{\text{\rm FOIL}}
\newcommand{\hfoil}{\text{\rm HL-FOIL}}

\newcommand{\dt}{\mathsf{DTree}} 
\newcommand{\obdd}{\mathsf{OBDD}} 
\newcommand{\fbdd}{\mathsf{FBDD}} 
 
\newcommand{\kcobdd}{k\text{-}\mathsf{COBDD}} 
\newcommand{\cnf}{\mathsf{CNF}}
\newcommand{\dnf}{\mathsf{DNF}}

\newcommand{\width}{\text{\rm width}}

\newcommand{\tox}[1]{\xrightarrow{#1}}

\newcommand{\dm}{\text{\rm dim}}

\tikzset{
    rt/.style={
		rectangle,
		fill = white,
		draw=black, 
		text centered,
		inner sep=0.5ex
		},
    rtt/.style={ 
    	rt,
    	inner sep=0.1ex
    	},
    ert/.style={ 
     	rt,
     	dashed
     	}, 
    ertt/.style={ 
        rtt,
        dashed
        }, 
    rect/.style={ 
        rectangle,
        fill = white,
        rounded corners,
        draw=black, 
        text centered,
        inner sep=0.8ex
        },
    rectw/.style={
        rect,
        draw=white
        },
    erect/.style={ 
    	rect,
    	dashed
    	},
    erectw/.style={ 
     	rectw,
     	dashed
     	},
    arrout/.style={
           ->,
           -latex,
           },
    arrin/.style={
           <-,
           latex-,
           },
    arrb/.style={
           <->,
           >=latex,
           }
}

\newcommand{\allpos}{\textsc{AllPos}}
\newcommand{\allneg}{\textsc{AllNeg}}
\newcommand{\expos}{\textsc{ExistsPos}}
\newcommand{\exneg}{\textsc{ExistsNeg}}
\newcommand{\match}{\textsc{Match}}

\newcommand{\pap}{\textsc{PartialAllPos}}
\newcommand{\pan}{\textsc{PartialAllNeg}}

\newcommand{\efoil}{\exists\foil}
\newcommand{\efoilp}{\efoil^+}
\newcommand{\afoil}{\forall\foil}
\newcommand{\afoilp}{\afoil^+}

\title{Foundations of Symbolic Languages \\ for Model Interpretability}

\author{
   Marcelo Arenas$^{1, 4}$, Daniel Baez$^{3}$, Pablo Barceló$^{2,4}$, Jorge Pérez$^{3,4}$, Bernardo Subercaseaux$^{4,5}$ \\ 
 	{ $^1$ Department of Computer Science, PUC-Chile}\\
  	{ $^2$ Institute for Mathematical and Computational Engineering, PUC-Chile }\\
  	{ $^3$ Department of Computer Science, Universidad de Chile}\\
  	{ $^4$ Millennium Institute for Foundational Research on Data, Chile} \\
	{ $^5$ Carnegie Mellon University, USA}\\
  	{\footnotesize  \texttt{[marenas, pbarcelo]@ing.puc.cl}, \texttt{jperez@dcc.uchile.cl}, \texttt{bsuberca@andrew.cmu.edu}}
}

\begin{document}

\maketitle

\begin{abstract}
Several queries and scores have 
been proposed to explain individual predictions made by ML models. 
Examples include queries based on “anchors”, which are parts of an instance that are sufficient to justify its classification, and “feature-perturbation” scores such as SHAP. 
Given the need for flexible, reliable, and easy-to-apply interpretability methods for ML models, 
we foresee the need for developing declarative languages to naturally specify different explainability queries. We do this in a principled way by rooting such a language in a logic called $\foil$, that allows for expressing many simple but important explainability queries, and might serve as a 
core for more expressive interpretability languages. 
We study the computational complexity of FOIL queries over classes of ML models often deemed to be easily interpretable: decision trees and more general decision diagrams. Since the number of possible inputs for an ML model is exponential in its dimension, tractability of the $\foil$ evaluation problem is delicate, but can be achieved by either restricting the structure of the models, or the fragment of 
$\foil$ being evaluated.  We also present a prototype implementation of $\foil$ wrapped in a high-level declarative language, and perform experiments showing that such a language can be used in practice.
\end{abstract}


\section{Introduction}
\todo[inline]{hola}
\label{sec:introduction}

{\bf Context.} 
The degree of {\em interpretability} of a machine learning (ML) model seems to be intimately related with the ability to ``answer questions'' about it. Those questions can either be global (behavior of the model as a whole) or local (behavior regarding certain instances/features). Concrete examples of such questions can be found 
in the recent literature, including, e.g., queries based on “anchors”, which are parts of an instance that are sufficient to justify its classification \cite{Ribeiro0G18,ChoiWD19,DarwicheH20,BarceloM0S20}, and numerical 
scores that measure the impact of the different features of an instance on its result \cite{StrumbeljK10,Ribeiro0G16,LundbergL17}. 

It is by now clear that ML interpretability admits no silver-bullet~\cite{doshivelez2017rigorous}, and that in many cases a combination of different queries 
may be the most effective way to understand a model's behavior. Also, 
model interpretability takes different flavors 
depending on the application domain one deals with. This naturally brings to the picture the need for
general-purpose specification languages that can provide flexibility and expressiveness to practitioners specifying interpretability queries. An even more advanced requirement for these languages is to be relatively easy to use in practice. This tackles the growing need for bringing interpretability methods closer to users 
with different levels of expertise.   

One way in which these requirements can be approached in a principled way is by developing a {\em declarative} interpretability language, i.e., 
one in which users directly 
express the queries they want to apply in the interpretability process (and not how these queries will be  evaluated). This is of course reminiscent of the path many
other areas in computer science have followed, in particular by using languages rooted in formal
logic; so has been the case, e.g., in data management \cite{AHV95}, knowledge representation \cite{DL-handbook}, and model checking \cite{mc-book}. 
One of the
advantages of this approach is that logics have a well-defined syntax and clear semantics. On the one hand, this ensures that the obtained explanations are provably sound and faithful to the model, which avoids a significant drawback of several techniques for explaining models in which the explanations can be inaccurate, or require  themselves to be further explained~\cite{abs-1811-10154}.
On the other hand, a logical root
facilitates the theoretical study of the computational cost of evaluation and optimization for the queries
in the language.

{\bf Our proposal.}
Our first contribution is the proposal of a logical language, 
called $\foil$, in which many simple yet relevant interpretability queries can be expressed. We 
believe that $\foil$ can further serve as a basis over which more expressive interpretability languages can be built, and we propose concrete directions of research towards its expansion. In a nutshell, given a decision model $\M$ that performs classification over instances $\es$ 
of dimension $n$, $\foil$ can express properties over the set of all {\em partial} and {\em full} instances of dimension $n$. 
A partial instance $\es$ is a vector of dimension $n$ in which some features are undefined. Such undefined features take a distinguished  value $\bot$. An instance is full if none of its features is undefined.    
The logic $\foil$ is simply first-order logic with access to two predicates on the set of all instances (partial or full) of dimension $n$: A unary predicate $\pos(\es)$, stating that $\es$ is a full instance 
that $\M$ classifies as positive, and a binary predicate 
$\es \subseteq \es'$, stating that instance $\es'$ potentially fills some of the undefined features of instance $\es$; e.g.,  $(1, \, 0, \, \bot) \subseteq (1, \, 0, \, 1)$, but $(1, \, 0, \, \bot) \not \subseteq (1, \, 1, \, 1)$.

As an overview of our proposal, consider the case of a bank using a binary model to judge applications for loans. Figure~\ref{fig:example-bank} illustrates the problem with concrete features, and Figure~\ref{fig:concreteSyntax} presents an example of a concrete interactive syntax.
In Figure~\ref{fig:concreteSyntax}, after loading and exploring the model,
the interaction asks whether the model could give a loan to a person who is married and does not have kids.
Assuming that the ``Accepted'' class is the positive one, this interaction can easily be formalized in $\foil$ by means of the query 
$\exists x \, \big(\pos(x) \land (\bot, \bot, 
 \bot, \bot, 0,  1,  \bot) \subseteq x \big). 
$

\begin{figure}
    \begin{subfigure}{0.5\textwidth}
      \centering
\scalebox{0.8}{
\begin{tikzpicture}

		\node at (-6, 0) (vn) {\small $\begin{matrix}
	   	\text{Stable job}\\
	   	>\text{ 40yo}\\
	   	\text{Previous loans}\\
	   	\text{Owns a house}\\
	  	\text{Has kids}\\
	  	\text{Married}\\
	  	\text{Criminal Record}
	   \end{matrix}$};

	   \node at (-4.5, 0) (v) {$\begin{pmatrix}
	   	0\\
	   	1\\
	   	1\\
	   	0\\
	   	0\\
	  	1\\
	  	1
	   \end{pmatrix}$};
		\node[draw, fill=black, text=white, text width=1.7cm, line width=.5mm] at (-2.25, 0) (bbm) {Black Box Model};
		\node[text width=3cm] at (1, 0) (ar) {Application\\ Rejected};
		
		\path[draw, ->, line width=1mm] (v) -- (bbm); 
		\path[draw, ->, line width=1mm] (bbm) -- (ar);
		
\end{tikzpicture}
}
      \caption{Diagram of a particular loan decision.}
      \label{fig:example-bank}
    \end{subfigure}
    \hspace{0.4em}
    \small
    \begin{subfigure}{0.5\textwidth}
      \begin{minted}{SQL}
> load("mlp.np") as MyModel;
> show features;
(stableJob, >40yo, prevLoan, ownsHouse,
hasKids, isMarried, crimRecord): Boolean
> show classes;
Rejected (0), Accepted (1) 

> exists person, 
    person.isMarried 
    and not person.hasKids
    and MyModel(person) = Accepted;
YES
      \end{minted}
      \caption{Example of a possible concrete syntax for a language tailored for interpretability queries.}
      \label{fig:concreteSyntax}
    \end{subfigure}
\caption{Example of a
bank that uses a model to decide whether to accept loan applications
considering binary features like 
``does the requester have a stable job'' and 
``are they older than 40''?}
\vspace{-4mm}
\end{figure}




{\bf Theoretical contributions.}
The {\em evaluation problem} for a fixed $\foil$ query $\varphi$ is as follows. Given a decision model $\M$, is it true that $\varphi$ is satisfied under the interpretation of predicates 
$\subseteq$ and $\pos$ defined above?  
An important caveat about this problem is that, in order to evaluate 
$\varphi$, we need to potentially look for an exponential number of 
instances, even if the features are Boolean, 
thus rendering the complexity of the problem infeasible in some cases. Think, for instance, of the query $\exists x \, \pos(x)$, which asks if 
$\M$ has at least one positive instance. Then this query is intractable for every class of models for which this problem 
is intractable; e.g., for the class of propositional formulas in CNF (notice that this is nothing but 
the {\em satisfiability problem} for the class at hand). 

The main theoretical contribution of our paper is an in-depth study of  the computational cost of $\foil$ 
on two classes of Boolean models that are often deemed to be ``easy to interpret'': decision trees and ordered binary decision diagrams 
(OBDDs) 
\cite{ChanD03,abs-1811-10154,GilpinBYBSK18,molnar2019,ChubarianT20}. 
An immediate advantage of these models over, say, CNF formulas, 
is that the satisfiability problem for them can be solved in polynomial time; i.e., the problem of evaluating the query $\exists x \pos(x)$ is tractable. Our study aims to  
(a) ``measure'' the    
degree of interpretability of said models in terms of the formal yardstick defined by the language $\foil$; and 
(b) shed light on when and how some simple interpretability queries can be evaluated efficiently on these decision models.

We start by showing that, in spite of the aforementioned
claims on the good level of interpretability for the  models considered, there is a simple 
query in $\foil$ that is intractable over them. 
In fact, such an intractable query 
has a natural 
``interpretability'' flavor, and thus
we believe this proof to be of independent interest. 

However, these intractability results should not immediately rule out the use of $\foil$ in practice. In fact, it is well known that a logic can be intractable in general, but become tractable in practically relevant cases. Such cases can be obtained by either restricting the syntactic fragment of the logic considered, or the structure of the models in which the logic is evaluated. We obtain positive results in both directions for the models we mentioned above. We explain them next. 

{\bf {\em Syntactic fragments.}}
We show that queries in $\efoil$, the existential fragment of $\foil$, admit tractable evaluation over the models we study. However, this language lacks expressive power for capturing some interpretability queries of practical interest. We then introduce $\efoil^+$, an extension of $\efoil$
with a finite set of unary universal queries from $\foil$ that is enough for expressing some relevant interpretability queries. We 
provide a
characterization theorem for the tractability of $\efoil^+$ over any class of Boolean decision models that reduces the tractability of this fragment to the tractability of two fixed and specific $\foil$ queries. Then we prove that said queries are tractable over perceptrons, which implies the tractability of $\efoil^+$ for this model. 
Unfortunately, the evaluation of said queries is $\np$-hard for decision trees and OBDDs.  Both the proof of tractability for perceptrons and intractability for decision trees and OBDDs are relatively simple, thus showing that the characterization theorem provides a useful technique for understanding which models are tractable for the evaluation of $\efoil^+$.

{\bf {\em Structural restrictions.}}
We restrict the models allowed in order to obtain tractability of evaluation for {\em arbitrary} $\foil$ queries. In particular, we show that evaluation of $\varphi$, for $\varphi$ a fixed $\foil$ query, can be solved in polynomial time over the class of OBDDs as long as they are {\em complete}, i.e., any path from the root to a leaf of the OBDD tests every feature from the input, and have bounded {\em width}, i.e., there is a constant bound on the number of nodes of the OBDD in which a feature can appear. 


{\bf Practical implementation.}
We designed \foil\ with a minimal set of logical constructs and tailored for models with binary input features. 
These decisions are reasonable for a detailed theoretical analysis but may hamper \foil\ usage in more general scenarios, in particular when models have (many) categorical or numerical input features, and queries 
are manually written by non-expert users.
To tackle this we introduce a more user-friendly language with a high-level syntax (\emph{à la} SQL in the spirit of the query in Figure~\ref{fig:concreteSyntax}) that can be compiled into \foil\ queries. 
Moreover, 
we present a prototype implementation that can be used to query decision trees trained in standard ML libraries by binarizing them into models (a subclass of 
binary decision diagrams) over which \foil\ queries can be efficiently evaluated.
We also test the performance of our implementation over synthetic and real data giving evidence of the usability of \foil\ as a base for practical interpretabilty languages.

\section{A Logic for Interpretability Queries}
\label{sec:logic}


{\bf Background.}
An {\em instance} of dimension $n$, with $n \geq 1$, is a tuple  $\es \in \{0,1\}^n$. We use notation $\es[i]$ to refer to the $i$-th component of this tuple, or equivalently, its $i$-th feature.
Moreover, we consider an abstract notion of a model of dimension $n$, and we define it as a Boolean function $\M : \{0,1\}^n \to \{0, 1\}$. That is, $\M$ assigns a Boolean value to each instance of dimension $n$, so that we focus on binary classifiers with Boolean input features. Restricting inputs and outputs to be Boolean makes our setting cleaner while still covering several relevant practical scenarios.  We use notation $\dm(\M)$ for the dimension of a model $\M$. 

A {\em partial instance} of dimension $n$ is a tuple $\es \in \{0,1,\bot\}^n$. Intuitively, if $\es[i] = \bot$, then the 
value of the $i$-th feature is undefined. Notice that an instance is a particular case of a partial instance where all features are assigned value either $0$ or $1$. Given two partial instances $\es_1$, $\es_2$ of dimension $n$, we say that $\es_1$ is {\em subsumed} by $\es_2$ if for every $i \in \{1, \ldots, n\}$ such that $\es_1[i] \neq \bot$, it holds that $\es_1[i] = \es_2[i]$. That is, $\es_1$ is subsumed by $\es_2$ if it is possible to obtain $\es_2$ from $\es_1$ by replacing some unknown values. Notice that a partial instance $\es$ can be thought of as a compact representation of the set of instances $\es'$ such that $\es$ is subsumed by $\es'$, where such instances $\es'$ are called the {\em completions} of $\es$.

{\bf Models.}
 A \emph{binary decision diagram}
(BDD~\cite{wegener2004bdds}) over instances of dimension $n$ is a rooted directed acyclic
graph~$\mathcal{M}$ with labels on edges and nodes such that: (i) each
leaf is labeled with~$\true$ or $\false$; (ii) each internal node (a
node that is not a leaf) is labeled with a feature $i \in \{1,\dots,n\}$; and
(iii) each internal node has two outgoing edges, one labeled~$1$ and
the another one labeled~$0$. Every instance $\es \in \{0,1\}^n$ defines a
unique path $\pi_\es = u_1 \cdots u_k$ from the root $u_1$ to a leaf
$u_k$ of $\M$ such that: if the label of
$u_i$ is $j \in \{1,\dots,n\}$, 
where $i \in \{1, \ldots, k-1\}$, then the edge from $u_i$ to $u_{i+1}$ is labeled with $\es[j]$. Moreover, the instance $\es$ is positive, denoted by
$\mathcal{M}(\es) = 1$, if the label of $u_k$
is~$\true$; otherwise the instance $\es$ is negative, which is denoted
by $\mathcal{M}(\es) = 0$. 
A binary decision
diagram~$\mathcal{M}$ is \emph{free} (FBDD) if for every path from the
root to a leaf, no two nodes on that path have the same label.
Besides, $\M$ is {\em ordered} (OBDD) if there exists a linear order $<$ on the set 
$\{1,\dots,n\}$ of features such that, 
if a node $u$ appears before a node $v$ in some path in $\M$ from the root to a leaf, then $u$ is labeled with $i$ and $v$ is labeled with $j$ for features $i,j$ such that 
$i < j$.
A \emph{decision tree} is simply an FBDD whose underlying DAG is a
tree. 
Finally, a perceptron $\M$ of dimension $n$ is a pair $(w, t)$ where $\vw \in \mathbb{R}^n$ and $t \in \mathbb{R}$, and the classification of an instance $\es \in \{0, 1\}^n$ is defined as $\M(\es) = 1$ if and only if 
$\vw \cdot \es  \geq t$.

In this paper, we focus on the following classes of models: $\obdd$, the class of ordered BDDs,  $\dt$, the class of decision trees, and $\perceptrons$, the class of perceptrons. None of these classes directly subsume the other: decision trees are not necessarily ordered, while the underlying DAG of an OBDD is not necessarily a tree. In fact, it is known that neither OBDDs can be compiled into polynomial-size decision trees nor decision trees into polynomial-size OBDDs \cite{Darwiche2002, Breitbart1995}. Perceptrons on the other hand can only model linear decision boundaries and thus are inherently less expressive than decision trees or OBDDs. It is also known that perceptrons cannot be compiled in polynomial time to decision trees or OBDDs unless $\mathrm{P} = \np$~\cite{BarceloM0S20}. 


{\bf The logic $\foil$.}
We consider first-order logic over a vocabulary consisting of a unary predicate $\pos$ and a binary predicate $\subseteq$. This logic is called {\em first-order interpretability logic} (\foil), and it is our reference language for defining conditions on models that we would like to reason about. In particular, predicate $\pos$ is used to indicate the value of an instance in a model, while predicate $\subseteq$ is used to represent the subsumption relation among partial instances. In what follows, we show that many natural properties can be expressed in a simple way in \foil, demonstrating the suitability of 
this language for the purpose of expressing explainability queries. 

We assume familiarity with the syntax and semantics of first-order logic (see the appendix 
for a review of these concepts). In particular, given a vocabulary $\sigma$ consisting of relations $R_1$, $\ldots$, $R_\ell$, recall that a structure $\astruct$ over $\sigma$ consists of a domain, where quantifiers are instantiated, and an interpretation for each relation $R_i$. Moreover, given a first-order formula $\varphi$ defined over the vocabulary $\sigma$, we write $\varphi(x_1, \ldots, x_k)$ to indicate that $\{x_1, \ldots, x_k\}$ is the set of free variables of $\varphi$. Finally, given a structure $\astruct$ over the vocabulary $\sigma$ and elements $a_1$, $\ldots$, $a_k$ in the domain of $\astruct$, we use $\astruct \models \varphi(a_1, \ldots, a_k)$ to indicate that formula $\varphi$ is satisfied by $\astruct$ when each variable $x_i$ is replaced by element $a_i$ ($1 \leq i \leq k$).

Our goal when introducing $\foil$ is to have a logic that allows to specify natural properties of models in a simple way. In this sense, we still need to define when a model $\M$ satisfies a formula in $\foil$, as $\M$ is not a structure over the vocabulary $\{\pos$, $\subseteq\}$ (so we cannot directly use 
the notion of satisfaction of a formula by a structure). More precisely, assuming that $\dm(\M) = n$, the structure $\astruct_\M$ associated to $\M$ is defined as follows. The domain of $\astruct_\M$ is the set $\{0,1, \bot\}^n$ of all partial instances of dimension $n$. An instance $\es \in \{0,1\}^n$ is in the interpretation of predicate $\pos$ in $\astruct_\M$ if and only if $\M(\es) = 1$. Finally, a pair $(\es_1,\es_2)$ is in the interpretation of predicate $\subseteq$ in $\astruct_\M$ if and only if $\es_1$ is subsumed by $\es_2$. Then, given a formula $\varphi(x_1, \ldots, x_k)$ in $\foil$ and partial instances $\es_1$, $\ldots$, $\es_k$ of dimension $n$, model $\M$ is said to {\em satisfy} $\varphi(\es_1, \ldots, \es_k)$, denoted by $\M \models \varphi(\es_1, \ldots, \es_k)$, if and only if $\astruct_\M \models \varphi(\es_1, \ldots, \es_k)$.

{\bf Evaluation problem.}
$\foil$ is our main tool in trying to understand how interpretable is a class of models. In particular, the following is the main problem studied in this paper, given a class $\C$ of models and a formula $\varphi(x_1, \ldots, x_k)$ in \foil.
\begin{center}
\fbox{\begin{tabular}{rl}
Problem: & \logicEvaluation$(\varphi, \mathcal{C})$\\
Input: & A model $\M \in \C$ of dimension $n$, and partial instances $\es_1, \ldots, \es_k$ of dimension $n$\\
Output: & \textsc{Yes}, if $\M \models \varphi(\es_1, \ldots, \es_k)$, and \textsc{No} otherwise
\end{tabular}}
\end{center}
For example, assume that $\cnf$, $\dnf$ are the classes of models given as propositional formulae in CNF and DNF, respectively. If $\varphi = \exists x \, \pos(x)$, then  $\logicEvaluation(\varphi, \cnf)$ is $\np$-complete and 
$\logicEvaluation(\varphi, \dnf)$ can be solved in polynomial time, as such problems correspond to the satisfiability problems for the propositional formulae in CNF and DNF, 
respectively. 

Given a model $\M$, it is important to notice that the size of the structure $\astruct_\M$ can be exponential in the size of $\M$. Hence, $\astruct_\M$ is a theoretical construction needed to formally define the semantics of $\foil$, but that should not be built when verifying in practice if a formula $\varphi$ is satisfied by $\M$. In fact, if we are 
aiming at finding tractable algorithms for $\foil$-evaluation, then
we need to design an algorithm 
that uses directly the encoding of $\M$ as a model (for example, as a binary decision tree) rather than as a logical structure. In other words, in order to evaluate a query $\varphi = \exists x \, \pos(x)$ over a model $\M$ of dimension $n$, one could certainly iterate over all $2^n$ instances $\es \in \{0, 1\}^n$  and evaluate $\M(\es)$. This of course impractical for even small-dimensional data. Therefore, evaluating formulas without iterating over the entire space of (partial) instances is the main technical challenge behind the results presented in this paper.



\section{Expressing Properties in the Logic}
\label{sec:examples}

{\bf Basic queries.}
We provide some formulas in $\foil$ to gain more insight into this logic. 
Fix a model $\M$ of dimension $n$. We can ask whether $\M$ assigns value 1 to some instance by using $\foil$-formula $\exists x \, \pos(x)$. Similarly,  formula $\exists y \, (\full(y) \wedge \neg \pos(y))$ can be used to check whether $\M$ assigns value 0 to some instance, where
\begin{eqnarray}
\label{eq-def-full}
\full(x) &=& \forall y \,(x \subseteq y \to x = y)
\end{eqnarray}
is used to verify whether all values in $x$ are known (that is, $\M \models \full(\es)$ if and only if $\es$ is an instance). Notice that formula $\full(y)$ has to be included in 
$\exists y \, (\full(y) \wedge \neg \pos(y))$
since $\M \models \neg \pos(\es)$ for each partial instance $\es$ with unknown values. 

Given an instance $\es$ such that $\M(\es) = 1$, we can ask if the values assigned to the first two features are necessary to obtain a positive classification. Formally, define $\es_{\{1,2\}}$ as a partial instance such that $\es_{\{1,2\}}[1] = \es_{\{1,2\}}[2] = \bot$ and $\es_{\{1,2\}}[i] = \es[i]$ for every $i \in \{3,\ldots,n\}$, and assume that
\begin{eqnarray*}
\varphi(x) &=& \forall y \, ((x \subseteq y  \wedge \full(y)) \to \pos(y)).
\end{eqnarray*}
If $\M \models \varphi(\es_{\{1,2\}})$, then the values assigned in $\es$ to the first two features are not necessary to obtain a positive classification. Notice that the use of unknown values in $\es_{\{1,2\}}$ is fundamental to reason about all possible assignments for the first two features, while keeping the remaining values of features unchanged. Besides, observe that a similar question can be expressed in $\foil$ for any set of features.

As before, we can ask if there is a completion of a partial instance $\es$ that is assigned value 1, by using $\foil$-formula $\psi(x) = \exists y \, (x \subseteq y  \wedge \full(y) \wedge \pos(y))$; that is, $\M \models \psi(\es)$ if and only if there is an assignment for the unknown values of $\es$ that results in an instance classified positively. 

{\bf Minimal sufficient reasons.}
Given an instance $\es$ and a partial instance $\es'$ that is subsumed by $\es$, consider the problem of verifying whether $\es'$ is a {\em sufficient reason} for $\es$ in the sense that every completion of $\es'$ is classified in the same way as $\es$~\cite{SCD18,BarceloM0S20, DBLP:journals/corr/abs-2010-11034}. The following query expresses this:
\begin{eqnarray}
\label{eq-def-sr}
\sr(x,y) & = & \full(x) \wedge y \subseteq x \wedge
\forall z \, [(y \subseteq z \wedge \full(z)) \to (\pos(x) \leftrightarrow \pos(z))],
\end{eqnarray}
given that $\M \models (\es, \es')$ if and only if $\es'$ is a sufficient reason for $\es$. Finally, it can also be expressed in $\foil$ the condition that $y$ is a {\em minimal} sufficient reason for $x$:
\begin{eqnarray*}
\msr(x,y) &=& \sr(x,y) \wedge \forall z \, ((z \subseteq y \wedge \sr(x,z)) \to z = y).
\end{eqnarray*}
That is, $\M \models (\es, \es')$ if and only if $\es'$ is a sufficient reason for $\es$, and there is no partial instance $\es''$ such that $\es''$ is a sufficient reason for $\es$ and $\es''$ is properly subsumed by $\es'$. Minimal sufficient reasons have also been called PI-explanations or abductive explanations in the literature~\cite{shih2018symbolic, DBLP:conf/nips/0001GCIN20, DBLP:journals/corr/abs-2010-11034, DBLP:journals/corr/abs-1811-10656}.

{\bf Bias detection queries.}
Let us consider an elementary approach to fairness based on 
\emph{protected} features, i.e., features from a set $P$ that 
should not be used for decision taking (e.g., gender, age, marital status, etc). We use a formalization of this notion proposed in \cite{DarwicheH20}, while noting it does not capture many other forms of biases and unfairness~\cite{10.1145/3457607}, and is thus to be taken only as an example.
Given a model $\M$ of dimension $n$, and a set of protected features $P \subseteq \{1, \ldots, n\}$, an instance $\es$ is said to be a {\em biased decision} of $\M$ if there exists an instance $\es'$ such that $\es$ and $\es'$ differ only on features from $P$ and $\M(\es) \neq \M(\es')$.
	A model $\M$ is {\em biased} if and only if there is an instance $\es$ that is a biased decision of $\M$. 
In what follows, we show how to encode queries relating to biased decisions in $\foil$. 

Let $S = \{1, \ldots, n\}$, and assume that $\vzero_S$ is an instance of dimension $n$ such that $\vzero_S[i] = 0$ for every $i \in S$, and $\vzero_S[j] = \bot$ for every $j \in \{1, \ldots, n\} \setminus S$. Moreover, define $\vone_S$ in the same way but considering value 1 instead of 0, and define
\begin{eqnarray*}
\match(x,y,u,v) & = & \forall z \, [(z \subseteq u \vee z \subseteq v) \to (z \subseteq x \leftrightarrow z \subseteq y)].
\end{eqnarray*}
When this formula is evaluated replacing $u$ by $\vzero_S$ and $v$ by $\vone_S$, it verifies whether $x$ and $y$ have the same value in each feature in $S$. More precisely, given a model $\M$ and instances $\es_1$, $\es_2$ of dimension $n$, we have that $\M \models \match(\es_1, \es_2, \vzero_S, \vone_S)$ if and only if $\es_1[i] = \es_2[i]$ for every $i \in S$. Notice that the use of free variables $u$ and $v$ as parameters allows us to represent the matching of two instances in the set of features $S$, as, in fact, such matching is encoded by the formula $\match(x,y,\vzero_S,\vone_S)$. The use of free variables as parameters is thus a useful feature of $\foil$.

With the previous terminology, we can define a query
\begin{multline*}
	\textsc{BiasedDecision}(x,u,v) \ = \ \full(x) \ \wedge \\ \exists y \, [\full(y) \wedge \match(x,y,u,v) \wedge (\pos(x) \leftrightarrow \neg \pos(y))].
\end{multline*}
To understand the meaning of this formula, assume that $N = \{1, \ldots, n\} \setminus P$ is the set of non-protected features. 
When $\textsc{BiasedDecision}(x,u,v)$ is evaluated replacing $u$ by $\vzero_N$ and $v$ by $\vone_N$, it verifies whether there exists an instance $y$ such that $x$ and $y$ have the same values in the non-protected features but opposite classification, so that $x$ is a biased decision. Hence, the formula
\begin{eqnarray*}
\textsc{BiasedModel}(u,v) & = & \exists x \, \textsc{BiasedDecision}(x,u,v)
\end{eqnarray*}
can be used to check whether a model $\M$ is biased with respect to the set $P$ of protected features, as $\M$ satisfies this property if and only if $\M \models \textsc{BiasedModel}(\vzero_N, \vone_N)$. 

A query of the form $\exists x \, \big(\pos(x) \land (\bot,  \bot, \bot, \bot, 0, 1, \bot ) \subseteq x \big)$ was included as an initial example in Section \ref{sec:introduction}.  According to the formal definition of $\foil$, such a query corresponds to $\varphi(u) = \exists x \, \big(\pos(x) \land u \subseteq x \big)$, and the desired answer is obtained when verifying whether $\varphi(\es)$ is satisfied by a model, where $\es[1] = \es[2] = \es[3] = \es[4] = \bot$, $\es[5] = 0$, $\es[6] = 1$ and $\es[7] = \bot$. Again, notice that the use of free variables as parameters is an important feature of $\foil$.

\section{Limits to Efficient Evaluation}
\label{sec:intractability}
Several important interpretability tasks have been shown to be tractable for the 
decision models we study in the paper~\cite{BarceloM0S20}, which has justified the informal claim that they 
are ``interpretable''. But this does not mean that all interpretability 
tasks are in fact tractable for these models. We try to formalize this idea by studying 
the complexity of evaluation for queries in $\foil$ over them. We show next
that the evaluation problem over the models studied in the paper can become intractable, even for some simple  
queries in the logic with a natural interpretability flavor.
This intractability result is of importance, in our view, as it sheds light on the limits of efficiency
for interpretability tasks over the models studied, and hence on the 
robustness of the folklore claims about them being ``interpretable''. 


\begin{theorem} \label{theo:main-dt} 
There exists a formula $\psi(x)$ in $\foil$ for which \logicEvaluation$(\psi(x), \dt)$ and 
\logicEvaluation$(\psi(x), \obdd)$ are $\np$-hard. 
\end{theorem}

This result tell us that there exists a concrete property expressible in $\foil$ that cannot be solved in polynomial time for decision trees and OBDDs (unless $\text{P} = \np$). 
In what follows, we describe this property, 
and how it is represented as a formula $\psi(x)$ in $\foil$ (the complete proof of Theorem \ref{theo:main-dt} is provided in the appendix).

Assume that $x \subset y$ is the formula $x \subseteq y \wedge x \neq y$ that verifies whether $x$ is properly subsumed by $y$.
We first define the following auxiliary predicates: 
\begin{eqnarray*}
\textsc{Adj}(x,y) &=&  
x \subset y \wedge \neg \exists z \, (x \subset z \wedge z \subset y),
\\
\textsc{Diff}(x,y) &=& \full(x) \wedge \full(y) \wedge x \neq y \wedge \exists z \, (\textsc{Adj}(z,x) \wedge \textsc{Adj}(z,y)).
\end{eqnarray*}
More precisely, $\textsc{Adj}(x,y)$ is used to check whether a partial instance $x$ is adjacent to a partial instance $y$, in the sense that  
$x$ is properly subsumed by $y$
and there is no partial instance $z$ such that $x$ is properly subsumed by $z$ and $z$ is properly subsumed by $y$. Moreover, $\textsc{Diff}(x,y)$ is used to verify whether two instances $x$ and $y$ differ exactly in the value of one feature. By using these predicates, we define the following notion of {\em stability} for an instance: 
\begin{eqnarray*}
\textsc{Stable}(x) &=& \forall y \, [\textsc{Diff}(x,y) \to (\pos(x) \leftrightarrow \pos(y))].
\end{eqnarray*}
That is, an instance $x$ is said to be stable if and only if any change in exactly one feature of $x$ leads to the same classification. Then the formula $\psi(x)$ in Theorem \ref{theo:main-dt} is defined as follows:
\begin{eqnarray*}
\psi(x) & = & \exists y \, (x \subseteq y \wedge \pos(y) \wedge \textsc{Stable}(y)).
\end{eqnarray*}
Hence, given a partial instance $x$, formula $\psi(x)$ is used to check if there is a completion of $x$ that is stable and positive. Theorem \ref{theo:main-dt} states that checking this for decision trees and OBDDs is an intractable problem. Observe that 
the notion of stability used in $\psi(x)$ has a natural interpretability flavor: it identifies positive instances whose classification is not affected by the perturbation of a single feature.  Note as well that the supposed interpretability of decision trees has already been questioned and nuanced in the literature~\cite{lip10.1145/3236386.3241340, BarceloM0S20}, to which this result contributes.


\section{Tractable Restrictions}
\label{sec:tractability}
Theorem \ref{theo:main-dt} tells us that evaluation of $\foil$ queries can be an intractable problem, but of course this does not completely rule 
out the applicability of the logic. In fact, as we show in this section one can obtain tractability by either restricting the analysis to a useful syntactic 
fragment of $\foil$, or by considering a structural restriction on the class of models over which $\foil$ queries are evaluated.  

\subsection{A tractable fragment of \textrm{FOIL}} 
\label{sec:tractable_foil}

We present a fragment of $\foil$ that is simple enough to yield tractability, but which is at the same time 
expressive enough to encode natural interpretability problems. This is not a trivial challenge, though, 
as the proof of Theorem \ref{theo:main-dt} 
shows 
intractability 
of queries 
in a syntactically simple fragment of 
$\foil$ (in fact, only two quantifier alternations 
suffice for the result to hold). 

Our starting point in this search is $\efoil$,  which is the fragment of $\foil$ consisting of all formulas where no universal quantifier occurs and no existential quantifier appears under a negation (each such a formula can be rewritten into a formula of the form $\exists x_1 \cdots \exists x_k \, \alpha$, where $\alpha$ does not mention any quantifiers). Moreover, we consider the fragment $\afoil$ of $\foil$, which is defined in the same way as $\efoil$ but exchanging the roles of universal and existential quantifiers. Then 
we show the~following:

\begin{proposition}
Let $\varphi$ be a query in $\efoil$ or $\afoil$. Then $\logicEvaluation(\varphi, \dt)$ and $\logicEvaluation(\varphi, \obdd)$ can be solved in polynomial time.
\label{prop:dt-obdd-efoil}
\end{proposition}

However, the fragment $\efoil$ has a limited expressive power since, for example, the predicate $\full(x)$ defined in (\ref{eq-def-full}) cannot be expressed in it (see Appendix \ref{sec:proof-full-efoil} for a formal proof of this claim).
To remedy this, 
we extend $\efoil$ by including predicate $\full(x)$ and two other unary predicates that are common in interpretability queries. More precisely, let 
$\allpos(x)$ and $\allneg(x)$ be unary predicates defined as follows:
\begin{eqnarray*}
\allpos(x) &=& \forall y \, \big((x \subseteq y \wedge \full(y)\big) \to \pos(y)),\\
\allneg(x) &=& \forall y \, \big(x \subseteq y \to \neg\pos(y)\big).
\end{eqnarray*}
Then $\efoil^+$ is defined as the fragment of $\foil$ consisting of all formulae where no universal quantifier occurs and no existential quantifier appears under a negation, and which are defined  over the extended vocabulary $\{\pos$, $\subseteq$, $\full$, 
$\allpos$, $\allneg\}$. In the same way, we define $\afoilp$ by exchanging the roles of universal and existential quantifiers.
Notice that the formula defining the notion of sufficient reason in (\ref{eq-def-sr}) 
is in $\afoilp$.
Similarly, the notion of minimal sufficient reason introduced in Section \ref{sec:examples} can be expressed in $\afoilp$:
\begin{multline*}
\msr(x,y) \ = \ \sr(x,y)  \wedge
\forall u \, [(u \subseteq y \wedge u \neq y \wedge \pos(x)) \to \neg \allpos(u)] \ \wedge\\
\forall v \, [(v \subseteq y \wedge v \neq y \wedge \neg \pos(x)) \to \neg \allneg(v)].
\end{multline*}
%
In what follows, we investigate the tractability of the fragments $\efoilp$ and $\afoilp$.  In particular, in the case of $\efoilp$, we show that the tractability for a class of models $\C$ can be characterized in terms of the tractability in $\C$ of two specific queries in $\efoilp$:
\begin{multline*}
\pap(x,y,z) \ = \
\exists u \, [x \subseteq u \wedge \allpos(u) \ \wedge\\
\exists v \, (y \subseteq v \wedge u \subseteq  v) \wedge \exists w \, (z \subseteq w \wedge u \subseteq w)],
\end{multline*}
and $\pan(x,y,z)$ that is defined exactly as $\pap(x,y,z)$ but replacing $\allpos(u)$ by $\allneg(u)$. More precisely, we have the following:
\begin{theorem}\label{theo:pap-pan}
For every class $\C$ of models, the following conditions are equivalent:
(a) $\logicEvaluation(\varphi ,\C)$ can be solved in polynomial time for each query $\varphi$ in $\efoilp$;
%
(b) $\logicEvaluation(\pap,\C)$ and $\logicEvaluation(\pan,\C)$ can be solved 
in polynomial~time.
\end{theorem}

This theorem gives us a concrete way to study the tractability of $\efoilp$ over a class of models. Besides, as the negation of a query in $\afoilp$ is a query in $\efoilp$, Theorem \ref{theo:pap-pan} also provides us with a tool to study the tractability of $\afoilp$. In fact, it is possible to prove the following for the class $\perceptrons$ of perceptrons.

\begin{proposition}\label{prop:pap-pan-ptron}
The problems
 $\logicEvaluation(\pap,\perceptrons)$ and
$\logicEvaluation(\pan,\perceptrons)$ can be solved in polynomial time.
\end{proposition}

From this proposition and Theorem \ref{theo:pap-pan}, it is possible to establish the following tractability results for $\efoilp$ and $\afoilp$. 

\begin{corollary}
Let $\varphi$ be a query in $\efoilp$ or $\afoilp$. Then $\logicEvaluation(\varphi ,\perceptrons)$ can be solved in polynomial time.
\end{corollary}

In fact, 
a
more general corollary holds: 
$\logicEvaluation(\varphi ,\perceptrons)$  is tractable as long as $\varphi$ is a {\em Boolean combination} of queries in $\efoil^+$ (which covers the case of $\afoil^+$).
Unfortunately, these queries turn out to be intractable over decision trees and OBDDs.

\begin{proposition}
Let $\C$ be $\obdd$ or $\dt$. The problems
 $\logicEvaluation(\pap,\C)$ and
$\logicEvaluation(\pan,\C)$ are $\np$-hard.
\label{prop:hardness-dt-ef}
\end{proposition}

\subsection{A structural restriction ensuring tractability}

We now look into the other direction suggested before, and identify a structural restriction on 
OBDDs that ensures tractability of 
evaluation for each query in $\foil$. This 
restriction is based on the usual notion of {\em width} of an OBDD
\cite{Bollig14,capelli_et_al:LIPIcs:2019:10257}. 
An OBDD $\M$ over a set $\{1,\dots,n\}$ of features 
is {\em complete} if each path from the root of $\M$ to one of its leaves includes every feature in $\{1,\dots,n\}$. The {\em width} of $\M$, denoted by $\width(\M)$, is defined as the maximum value $n_i$ for $i \in \{1, \ldots, n\}$, where $n_i$ is the number of nodes of $\M$ labeled by feature $i$.
Then, given $k \geq 1$, 
$\kcobdd$ is defined as
the class of complete OBDDs $\M$ such that $\width(\M) \leq k$. By building on techniques from \cite{capelli_et_al:LIPIcs:2019:10257}, we prove that:
\begin{theorem}\label{theo:kcobdd}
Let $k \geq 1$ and query $\varphi$ in $\foil$. Then 
$\logicEvaluation(\varphi,\kcobdd)$ 
can be solved in polynomial~time.
\end{theorem}




\section{Practical Implementation}
\label{sec:implementation}
The \foil\ language has at least two downsides from a usability point of view.
First, in \foil\ every query is constructed using a minimal set of basic logical constructs.
Moreover, the variables 
in queries are instantiated by feature vectors that may have hundreds of components.
This implies that some simple queries may need fairly long and complicated \foil\ expressions.
Second, \foil\ is designed to only work over models with binary input features.
These downsides are a consequence of our design decisions that were reasonable for a detailed theoretical analysis but may hamper \foil\ usage in more general scenarios, in particular when models have (many) categorical or numerical input features.

In this section, we describe a simple high-level syntax and implementation of a more user-friendly language (\emph{à la} SQL) to query general decision trees, and we show how to compile it into \foil\ queries to be evaluated over a suitable binarization of the queried model.
As a whole, the pipeline requires several pieces that we explain in this section: (i) a working and efficient query-evaluation implementation of a fragment of \foil\  over a suitable sub-class of Binary Decision Diagrams (BDDs), (ii) a transformation from the high-level syntax to \foil\ queries, and (iii) a transformation from a general decision tree to a BDD 
over which the \foil\ query can be efficiently evaluated.
%
We only present here the main ideas and intuitions of the implemented methods. 
A detailed exposition along with our implementation and a set of real examples can be found in the supplementary material.


\subsection{Implementing and testing core \foil}
\label{sec:core-implementation}

We implemented 
a version of the algorithm derived from Section \ref{sec:tractable_foil} for evaluating existential and universal \foil\ queries that is proven to work over a suitable sub-class of
BDDs. 
The method receives a query as a plain text file and a BDD in JSON format.
We tested the efficiency of our implementation varying three different parameters: the number of input features, the number of leaves of the decision tree, and the size of the input queries. 
We created a set of trees trained with random input data with input feature dimensions in the range $[10,350]$, and of $100$, $500$ and $1000$ leaves ($24$ different decision trees).
We note that the best performing decision trees over standard datasets~\cite{OpenML} rarely contain more than $1000$ total nodes~\cite{DTs}, thus the trees that we tested can be considered of standard size. 
We created a set of random queries with $1$ to $4$ quantified variables, and a varying number of operators ($60$ different queries).
We run every query $5$ times over each tree, and averaged the execution time to obtain the running time of one case.
From all our tests no case required more than $2.5$ seconds for its complete evaluation with a total average execution time of $0.213$ seconds and standard deviation of $0.169$ in the whole dataset. 
Figure~\ref{fig:avg_time} shows the average time (average over different queries) for all settings. 
We observed that some queries where specially more time consuming than others.
Figure~\ref{fig:max_time} shows the maximum execution time over all queries for each setting.
The most important factor when evaluating queries is the number of input features, which is consistent with a theoretical worst case analysis. All experiments where run on a personal computer with a 2.48GHz Intel N3060 processor and 2GB RAM. The exact details of the machine are presented in the supplementary material.
\begin{figure}
    \begin{subfigure}{0.3\textwidth}
    \includegraphics[width=0.95\linewidth]{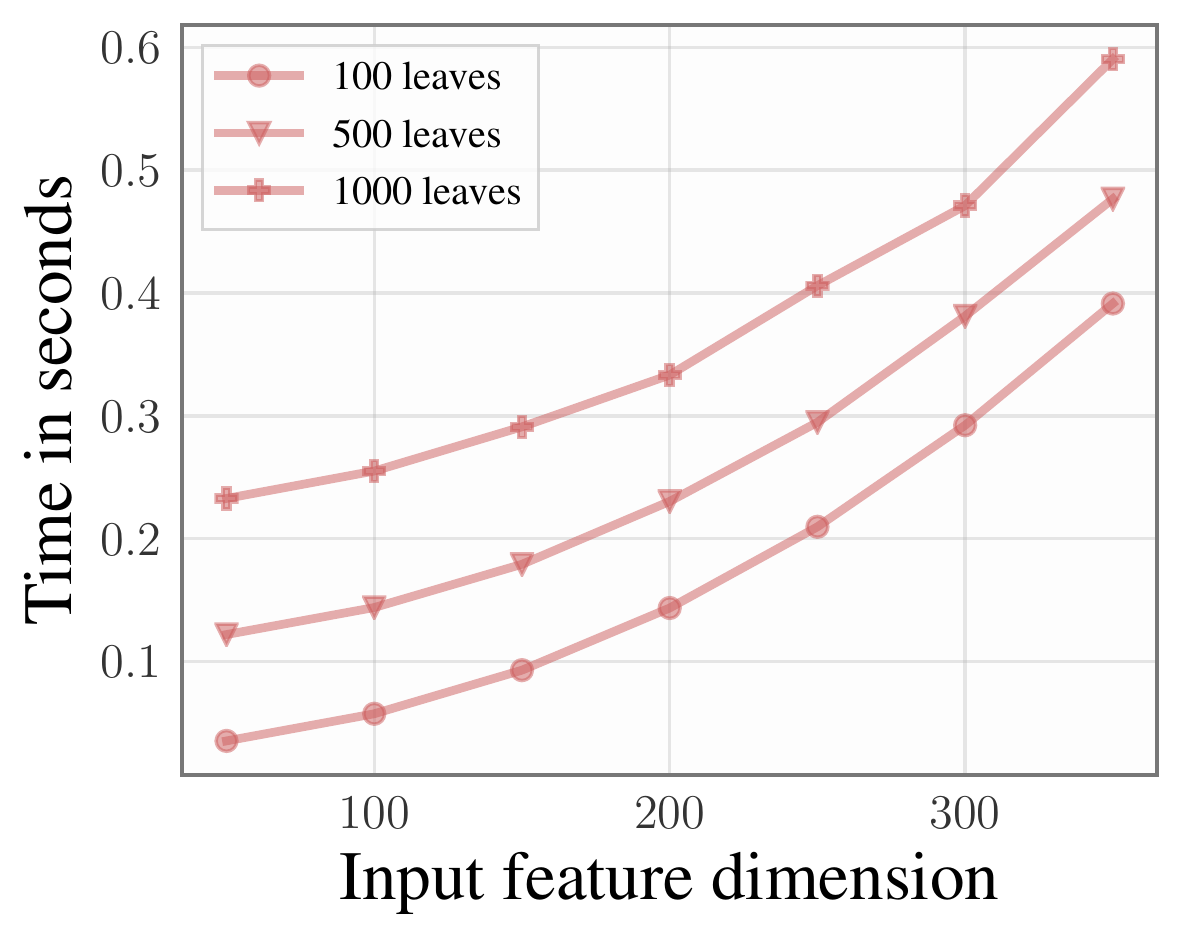}
    \caption{Average time for 60 random \foil\ queries over Decision Trees trained with random data.}
    \label{fig:avg_time}
    \end{subfigure}
    \hspace*{5pt}
    \begin{subfigure}{0.3\textwidth}
    \includegraphics[width=0.95\linewidth]{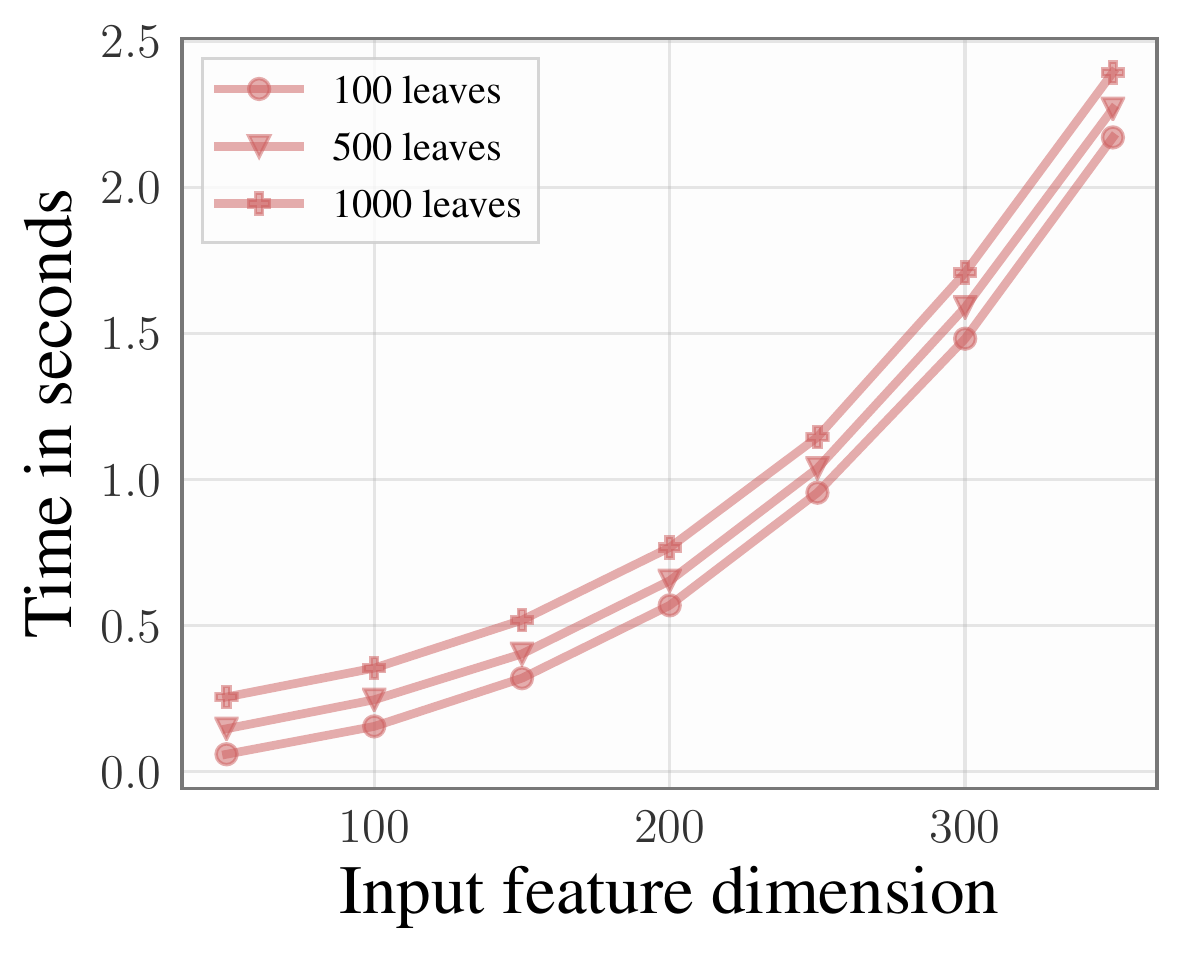}
    \caption{Maximum time for 60 random \foil\ queries over Decision Trees trained with random data.}
    \label{fig:max_time}
    \end{subfigure}
    \hspace*{5pt}
    \begin{subfigure}{0.3\textwidth}
    \small
    \begin{minted}{SQL}
    
> exists student, 
    student.age <= 18 and 
    (student.internetAtHome or
     student.male) and
    goodGrades(student)
    \end{minted}
    \caption{Example of a query in our system executed over a model trained in the dataset in~\cite{student}.}
    \label{fig:real_time}
    \end{subfigure}
    \caption{Execution time for \foil\ queries and a high-level practical syntax.}
    \label{fig:time}
    \vspace{-1mm}
\end{figure}

\subsection{Interpretability symbolic queries in practice}
\label{sec:user}


{\bf High-level features.}
We designed and implemented a prototype system for user-friendly interpretability queries.
Figure~\ref{fig:real_time} shows a real example query that can be posed in our system for a model trained over the \emph{Student Performance Data Set}~\cite{student}.
Notice that our syntax allow named features, names for the target class (\texttt{goodGrades} in the example) and the comparison with numerical thresholds which goes beyond the \foil\ formalization.
Our current implementation allows for numerical and logical comparisons, as well as handy logical shortcuts such as \texttt{implies} and \texttt{iff}. 
Moreover we implemented a wrapper to directly import Decision Trees trained in the Scikit-learn~\cite{scikit-learn} library.
{\bf Binarizing models and queries.}
One of the main issues when compiling these new queries into \foil\ is how to binarize numerical features.
Choi et al.~\cite{DBLP:journals/corr/abs-2007-01493} 
describe in extensive detail an approach to encode general decision trees into binary ones.
The key observation is that one can separate numerical values into equivalence classes depending on the thresholds used by a decision tree.
For example, assume a tree with an \emph{age} feature that learns nodes with thresholds $\text{age} \leq 16$ and $\text{age} \leq 24$.
It is clear that such a tree cannot distinguish an $\text{age}=17$ from an $\text{age}=19$. 
In general, every tree induces a finite number of equivalence classes for each numerical feature and one can take advantage of that to produce a binary version of the tree~\cite{DBLP:journals/corr/abs-2007-01493}.
In our case, we also need to take the query into account.
For instance, when evaluating a query with a condition \texttt{student.age <= 18}, ages $17$ and $19$ become distinguishable.
Considering all  these thresholds we have intervals $(- \infty, 16]$, $(16, 18]$, $(18,24]$, $(24, \infty)$ and we can use four binary features to encode in which interval an age value lies.
It is worth noting that this process creates extra artificial features, and thus, the decision tree that learned real thresholds needs to be binarized in the new feature space accordingly.
One can show that a naive implementation would imply an exponential blow up in the size of the new tree.
To avoid this our binarization process transforms the real-valued decision tree into a binary FBDD, over which we prove that our polynomial algorithms from Section~\ref{sec:tractable_foil} are still applicable.

{\bf Performance tests.}
We tested a set of $20$ handcrafted queries over decision trees with up to 400 leaves trained for the Student Performance Data Set~\cite{student}, which combines Boolean and numerical features.  
Our results show that natural queries can be evaluated over decision trees of standard size~\cite{DTs} in less than a second on a standard personal machine,
thus validating the practical usability of our prototype.

\section{Final Remarks and Future Work}
\label{sec:final}
In several aspects the logic $\foil$ is limited in expressive power for interpretability purposes. This was a design decision for this paper, 
in order to start with a ``minimal'' logic that would allow highlighting the benefits of having a declarative language for interpretability tasks, and at the same time allowing to carry out a clean theoretical analysis of its evaluation complexity. However, a genuinely practical declarative language should include other functionalities 
that allow more sophisticated queries to be expressed. As an example, consider the notion of SHAP-score \cite{LundbergL17} that has a predominant place in the literature on interpretability issues today. In a nutshell, for a decision model $\M$ with $\dim(\M) = n$ 
and instance $\es \in \{0,1\}^n$, this score corresponds to a weighted sum of expressions of the form $\# \pos_S(\es)$, for $S \subseteq \{1,\dots,n\}$, where $\# \pos_S(\es)$ is the number of instances $\es'$ for which $\M(\es') = 1$ and $\es'$ coincides with $\es$ over all features in $S$. Expressing this query, hence, requires extending $\foil$ with a recursive mechanism that permits to iterate over the subsets $S$ of $\{1,\dots,n\}$, and a feature for counting the number of positive completions of a partial instance; e.g., in the form of a ``numerical'' query $\phi(x) := 
\# y. (x \subseteq y \wedge \pos(y))$. Logics of this kind abound in computer science logic (c.f., \cite{Libkin04,arenasCounting}), and one could use all this knowledge in order to build a suitable extension of $\foil$ for dealing with
this kind of interpretability tasks.    
One can also envision a language facilitating the comparison of different models by providing separate $\pos$ predicates for each of them. Then, for example, one can ask whether two models are equivalent, or if they differ for a particular kind of instances. 
Such an extension can affect the complexity of evaluation in nontrivial ways. 


Arguably, interpretability measures the degree in which \emph{humans} can understand decisions made by \emph{machines}.
One of our main calls in this paper is to build more \emph{symbolic} interpretability tools, and thus, make them closer to how humans reason about facts and situations.
Having a symbolic high-level interpretability language to inspect ML models and their decisions is thus a natural and challenging way of pursuing this goal.
We took a step further in this paper presenting theoretical and practical results, but several problems remain open.
A particularly interesting one is whether a logical language can effectively interact with intrinsically non-symbolic models, and if so, what mechanisms could allow for practical tractability without sacrificing provable correctness.


\begin{acksection}
This work was partially funded by ANID
- Millennium Science Initiative Program - Code ICN17\_002. Arenas is funded by Fondecyt grant
1191337, while Barcel\'o and P\'erez are funded by Fondecyt grant 1200967.
\end{acksection}

\bibliographystyle{abbrv}
\bibliography{main}

\begin{thebibliography}{10}

\bibitem{AHV95}
S.~Abiteboul, R.~Hull, and V.~Vianu.
\newblock {\em Foundations of Databases}.
\newblock Addison-Wesley, 1995.

\bibitem{arenasCounting}
M.~Arenas, M.~Mu{\~n}oz, and C.~Riveros.
\newblock Descriptive complexity for counting complexity classes.
\newblock {\em Logical Methods in Computer Science ; Volume 16}, pages Issue 1
  ; 1860--5974, 2020.

\bibitem{DL-handbook}
F.~Baader, D.~Calvanese, D.~L. McGuinness, D.~Nardi, and P.~F.
  Patel{-}Schneider, editors.
\newblock {\em The Description Logic Handbook: Theory, Implementation, and
  Applications}.
\newblock Cambridge University Press, 2003.

\bibitem{BarceloM0S20}
P.~Barcel{\'{o}}, M.~Monet, J.~P{\'{e}}rez, and B.~Subercaseaux.
\newblock Model interpretability through the lens of computational complexity.
\newblock In {\em NeurIPS}, 2020.

\bibitem{Bollig14}
B.~Bollig.
\newblock On the width of ordered binary decision diagrams.
\newblock In {\em COCOA}, pages 444--458, 2014.

\bibitem{Breitbart1995}
Y.~Breitbart, H.~Hunt, and D.~Rosenkrantz.
\newblock On the size of binary decision diagrams representing boolean
  functions.
\newblock {\em Theoretical Computer Science}, 145(1-2):45--69, July 1995.

\bibitem{10.1109/TC.1986.1676819}
R.~E. Bryant.
\newblock Graph-based algorithms for boolean function manipulation.
\newblock {\em IEEE Trans. Comput.}, 35(8):677–691, Aug. 1986.

\bibitem{DBLP:journals/corr/abs-1807-04263}
F.~Capelli and S.~Mengel.
\newblock Knowledge compilation, width and quantification.
\newblock {\em CoRR}, abs/1807.04263, 2018.

\bibitem{capelli_et_al:LIPIcs:2019:10257}
F.~Capelli and S.~Mengel.
\newblock {Tractable QBF by Knowledge Compilation}.
\newblock In {\em STACS}, pages 18:1--18:16, 2019.

\bibitem{ChanD03}
H.~Chan and A.~Darwiche.
\newblock Reasoning about bayesian network classifiers.
\newblock In {\em {UAI}}, pages 107--115, 2003.

\bibitem{chang1990model}
C.~C. Chang and H.~J. Keisler.
\newblock {\em Model theory}.
\newblock Elsevier, 1990.

\bibitem{DBLP:journals/corr/abs-2007-01493}
A.~Choi, A.~Shih, A.~Goyanka, and A.~Darwiche.
\newblock On symbolically encoding the behavior of random forests.
\newblock {\em CoRR}, abs/2007.01493, 2020.

\bibitem{ChoiWD19}
A.~Choi, R.~Wang, and A.~Darwiche.
\newblock On the relative expressiveness of bayesian and neural networks.
\newblock {\em Int. J. Approx. Reason.}, 113:303--323, 2019.

\bibitem{ChubarianT20}
K.~Chubarian and G.~Tur{\'{a}}n.
\newblock Interpretability of bayesian network classifiers: {OBDD}
  approximation and polynomial threshold functions.
\newblock In {\em ISAIM}, 2020.

\bibitem{mc-book}
E.~M. Clarke, O.~Grumberg, and D.~A. Peled.
\newblock {\em Model checking}.
\newblock {MIT} Press, 2001.

\bibitem{DarwicheH20}
A.~Darwiche and A.~Hirth.
\newblock On the reasons behind decisions.
\newblock In {\em {ECAI}}, pages 712--720, 2020.

\bibitem{Darwiche2002}
A.~Darwiche and P.~Marquis.
\newblock A knowledge compilation map.
\newblock {\em Journal of Artificial Intelligence Research}, 17:229--264, Sept.
  2002.

\bibitem{doshivelez2017rigorous}
F.~Doshi-Velez and B.~Kim.
\newblock Towards a rigorous science of interpretable machine learning, 2017.

\bibitem{GilpinBYBSK18}
L.~H. Gilpin, D.~Bau, B.~Z. Yuan, A.~Bajwa, M.~Specter, and L.~Kagal.
\newblock Explaining explanations: An overview of interpretability of machine
  learning.
\newblock In {\em DSAA}, pages 80--89, 2018.

\bibitem{DBLP:journals/corr/abs-1811-10656}
A.~Ignatiev, N.~Narodytska, and J.~Marques{-}Silva.
\newblock Abduction-based explanations for machine learning models.
\newblock {\em CoRR}, abs/1811.10656, 2018.

\bibitem{DBLP:journals/corr/abs-2010-11034}
Y.~Izza, A.~Ignatiev, and J.~Marques{-}Silva.
\newblock On explaining decision trees.
\newblock {\em CoRR}, abs/2010.11034, 2020.

\bibitem{Libkin04}
L.~Libkin.
\newblock {\em Elements of Finite Model Theory}.
\newblock Texts in Theoretical Computer Science. An {EATCS} Series. Springer,
  2004.

\bibitem{lip10.1145/3236386.3241340}
Z.~C. Lipton.
\newblock The mythos of model interpretability: In machine learning, the
  concept of interpretability is both important and slippery.
\newblock {\em Queue}, 16(3):31–57, June 2018.

\bibitem{LundbergL17}
S.~M. Lundberg and S.~Lee.
\newblock A unified approach to interpreting model predictions.
\newblock In {\em NIPS}, pages 4765--4774, 2017.

\bibitem{DTs}
R.~G. Mantovani, T.~Horv{\'{a}}th, R.~Cerri, S.~B. Junior, J.~Vanschoren, and
  A.~C.~P. de~Leon Ferreira~de Carvalho.
\newblock An empirical study on hyperparameter tuning of decision trees.
\newblock {\em CoRR}, abs/1812.02207, 2018.

\bibitem{DBLP:conf/nips/0001GCIN20}
J.~Marques{-}Silva, T.~Gerspacher, M.~C. Cooper, A.~Ignatiev, and
  N.~Narodytska.
\newblock Explaining naive bayes and other linear classifiers with polynomial
  time and delay.
\newblock In H.~Larochelle, M.~Ranzato, R.~Hadsell, M.~Balcan, and H.~Lin,
  editors, {\em Advances in Neural Information Processing Systems 33: Annual
  Conference on Neural Information Processing Systems 2020, NeurIPS 2020,
  December 6-12, 2020, virtual}, 2020.

\bibitem{10.1145/3457607}
N.~Mehrabi, F.~Morstatter, N.~Saxena, K.~Lerman, and A.~Galstyan.
\newblock A survey on bias and fairness in machine learning.
\newblock {\em ACM Comput. Surv.}, 54(6), July 2021.

\bibitem{molnar2019}
C.~Molnar.
\newblock {\em Interpretable Machine Learning}.
\newblock 2019.
\newblock \texttt{https://christophm.github.io/interpretable-ml-book/}.

\bibitem{student}
F.~Pagnotta and H.~M. Amran.
\newblock Using data mining to predict secondary school student alcohol
  consumption, 2016.

\bibitem{scikit-learn}
F.~Pedregosa, G.~Varoquaux, A.~Gramfort, V.~Michel, B.~Thirion, O.~Grisel,
  M.~Blondel, P.~Prettenhofer, R.~Weiss, V.~Dubourg, J.~Vanderplas, A.~Passos,
  D.~Cournapeau, M.~Brucher, M.~Perrot, and E.~Duchesnay.
\newblock Scikit-learn: Machine learning in {P}ython.
\newblock {\em Journal of Machine Learning Research}, 12:2825--2830, 2011.

\bibitem{Ribeiro0G16}
M.~T. Ribeiro, S.~Singh, and C.~Guestrin.
\newblock "why should {I} trust you?": Explaining the predictions of any
  classifier.
\newblock In {\em SIGKDD}, pages 1135--1144, 2016.

\bibitem{Ribeiro0G18}
M.~T. Ribeiro, S.~Singh, and C.~Guestrin.
\newblock Anchors: High-precision model-agnostic explanations.
\newblock In {\em AAAI}, pages 1527--1535, 2018.

\bibitem{abs-1811-10154}
C.~Rudin.
\newblock Please stop explaining black box models for high stakes decisions.
\newblock {\em CoRR}, abs/1811.10154, 2018.

\bibitem{SCD18}
A.~Shih, A.~Choi, and A.~Darwiche.
\newblock A symbolic approach to explaining bayesian network classifiers.
\newblock In J.~Lang, editor, {\em Proceedings of the Twenty-Seventh
  International Joint Conference on Artificial Intelligence, {IJCAI} 2018, July
  13-19, 2018, Stockholm, Sweden}, pages 5103--5111, 2018.

\bibitem{shih2018symbolic}
A.~Shih, A.~Choi, and A.~Darwiche.
\newblock A symbolic approach to explaining bayesian network classifiers, 2018.

\bibitem{StrumbeljK10}
E.~Strumbelj and I.~Kononenko.
\newblock An efficient explanation of individual classifications using game
  theory.
\newblock {\em J. Mach. Learn. Res.}, 11:1--18, 2010.

\bibitem{OpenML}
J.~Vanschoren, J.~N. Van~Rijn, B.~Bischl, and L.~Torgo.
\newblock Openml: networked science in machine learning.
\newblock {\em ACM SIGKDD Explorations Newsletter}, 15(2):49--60, 2014.

\bibitem{wegener2004bdds}
I.~Wegener.
\newblock Bdds: design, analysis, complexity, and applications.
\newblock {\em Discrete Applied Mathematics}, 138(1-2):229--251, 2004.

\end{thebibliography}

\newpage

\appendix
\begin{center}
   {\huge \textbf{Appendix}}
\end{center} 

\paragraph{Organization}

The supplementary material is organized as follows: Section~\ref{sec:review_fo} presents a brief review of the concepts concerning first-order logic that are used in our work. Section~\ref{sec:evaluation} presents a proof of Theorem~\ref{theo:main-dt}, our negative result concerning decision trees and OBDDs, while Section~\ref{sec:tractable-efoil} is devoted to our positive result. Section~\ref{sec:proof-full-efoil} proves that $\efoilp$ is strictly more expressive than $\efoil$, justifying its independent study. Then, Section~\ref{sec:tractable-efoilp} is devoted to the tractability of $\efoil$ and $\efoilp$; it includes proofs both for Theorem~\ref{theo:pap-pan} and Proposition~\ref{prop:pap-pan-ptron}, which together imply the tractability of $\efoilp$ for perceptrons.  Next, Section~\ref{sec:tractable-kcobdd} presents a proof of Theorem~\ref{theo:kcobdd}, implying the full tractability of $\foil$ for a restricted class of OBDDs.  Section~\ref{sec:supp-implementation} discusses details of the practical implementation, while Section~\ref{sec:experimentation} explains the the methodology of our experiments. Then, Section~\ref{sec:high-level} discusses details of the high-level version we implemented, and also presents several examples of queries for the \emph{Student Performance Data Set} which serve to show the usability of our implementation in practice. Finally Section~\ref{sec:binarization} explains the binarization process for real-valued decision trees and high-level queries.
A repository with code with our implementation for \foil\ and the high-level syntax as well as examples and scripts to replicate our experiments can be found at
\begin{center}
\url{https://github.com/AngrySeal/FOIL-Prototype}
\end{center}

\section{Syntax and semantics of first-order logic} 
\label{sec:review_fo}

We review the definition of first-order logic (\FO) over vocabularies consisting only of relations.

\paragraph{Syntax of \FO.} A {\em vocabulary} $\sigma$ is a finite set $\{R_1,\dots,R_m\}$, where each $R_i$ is a relation symbol with associated arity $n_i > 0$, for $i \in \{1, \ldots, m\}$.
We assume the existence of a countably infinite set of variables $\{x,y,z,\dots\}$, possibly with subscripts. 
The set 
of {\em \FO-formulas over $\sigma$} is inductively defined as follows. 
\begin{enumerate}
\item If $x,y$ are variables,
then $x = y$ is an $\FO$-formula over $\sigma$. 
\item If relation symbol $R \in \sigma$ has arity $n > 0$ and $x_1,\dots,x_n$ are 
variables, then $R(x_1,\dots,x_n)$ is an $\FO$-formula over $\sigma$.
\item If $\varphi,\psi$ are $\FO$-formulas over $\sigma$,
then $(\neg \varphi)$, $(\varphi \vee \psi)$, and $(\varphi \wedge \psi)$ are $\FO$-formulas over $\sigma$.
\item If $x$ is a variable and $\varphi$ is an $\FO$-formula over $\sigma$,
then $(\exists x \, \varphi)$ and $(\forall x \, \varphi)$ are $\FO$-formulas over $\sigma$.
\end{enumerate}
\FO-formulas of type (1) and (2) are called {\em atomic}. A variable $x$ in $\FO$-formula $\varphi$ appears {\em free}, if there is an occurrence of $x$ in $\varphi$ that is not in the scope of a quantifier 
$\exists x$ or $\forall x$. 
An {\em $\FO$-sentence}  is an $\FO$-formula without free variables. We often write $\varphi(x_1, \ldots, x_k)$ to denote that $\{x_1, \ldots, x_k\}$ is the set of free variables of 
$\varphi$. 

\paragraph{Semantics of \FO.} $\FO$-formulae over a vocabulary $\sigma$ are interpreted over {\em $\sigma$-structures}. Formally, a $\sigma$-structure is 
a tuple 
\begin{eqnarray*}
\astruct & = & \langle A,\,R_1^\astruct,\cdots,R_m^\astruct\rangle,
\end{eqnarray*}
where $A$ is the {\em domain} of $\astruct$, and for each relation symbol $R \in \sigma$ of arity $n$,
we have that $R^\astruct$ is an $n$-ary relation over $A$. 
We call $R_i^\astruct$ 
the {\em interpretation} of $R_i$ 
in $\astruct$. 

Let $\varphi$ be an $\FO$-formula over a vocabulary $\sigma$, and 
$\astruct$ a $\sigma$-structure.
Consider a mapping $\nu$ that associates an element in $A$ to each variable.
We formally define the {\em satisfaction of $\FO$-formula $\varphi$ over the pair $(\astruct,\nu)$}, denoted by $(\astruct,\nu) \models \varphi$, as follows.  
\begin{enumerate}
\item If $\varphi$ is an atomic formula of the form $x = y$, 
then 
$(\astruct,\nu) \models \varphi \, \Leftrightarrow \, \nu(t_1) = \nu(t_2)$. 

\item If $\varphi$ is an atomic formula of the form $R(x_1,\ldots,x_n)$ for some $R \in \sigma$, 
then 
$(\astruct,\nu) \models \varphi \, \Leftrightarrow \, (\nu(x_1),\dots,\nu(x_n)) \in R^\astruct$. 

\item If $\varphi$ is of the form $(\neg \psi)$, then 
$(\astruct,\nu) \models \varphi \, \Leftrightarrow \, (\astruct,\nu) \not\models \psi$.
 \item If $\varphi$ is of the form $(\psi \vee \psi')$, then 
$(\astruct,\nu) \models \varphi$ iff $(\astruct,\nu) \models \psi$ or $(\astruct,\nu) \models \psi'$.
 \item If $\varphi$ is of the form $(\psi \wedge \psi')$, then 
$(\astruct,\nu) \models \varphi$ iff $(\astruct,\nu) \models \psi$ and $(\astruct,\nu) \models \psi'$.
\item If $\varphi$ is of the form $(\exists x \, \psi)$, then $(\astruct,\nu) \models \varphi$ iff there exists $a \in A$ for which 
$(\astruct,\nu[x/a]) \models \psi$. Here, $\nu[x/a]$ is a mapping that takes the same value as $\nu$ 
on every variable $y \neq x$, and takes value $a$ on $x$. 
\item If $\varphi$ is of the form $(\forall x \, \psi)$, then $(\astruct,\nu) \models \varphi$ iff for every $a \in A$ we have that 
\mbox{$(\astruct,\nu[x/a]) \models \psi$.}
\end{enumerate}
For an \FO-formula $\varphi(x_1,\dots,x_k)$ and assignment $\nu$ such that $\nu(x_i) = a_i$, for each $i \in \{1, \ldots, k\}$,  
we  write $\astruct \models \varphi(a_1, \ldots, a_k)$
to denote that $(\astruct,\nu) \models \varphi$. 
If $\varphi$ is a sentence, we write simply $\astruct \models \varphi$, as for any pair of mappings $\nu_1$, $\nu_2$ for the variables, it holds that $(\astruct,\nu_1) \models \varphi$ iff
$(\astruct,\nu_2) \models \varphi$.

\section{Proof that the FULL predicate cannot be expressed in the existential fragment of FOIL}
\label{sec:proof-full-efoil}


This proof requires some background in model theory. Namely, it uses the following ideas:

\begin{itemize}
    \item Given a structure $\mathfrak{A}$ with domain $A$, a set $S \subseteq A$ induces a \emph{sub-structure} $\mathfrak{A}'$ such that the domain of $\astruct'$ is $A'$ and $R^{\astruct'} = R^{\astruct} \cap S^n$ for every relation $R \in \sigma$ of arity $n$.
    \item Let $\mathfrak{A}$, $\mathfrak{B}$ be two structures over a vocabulary $\sigma$ with domains $A$ and $B$, respectively.
    An \emph{isomorphism} $f: A \to B$ between $\mathfrak{A}$ and $\mathfrak{B}$ satisfies the following property for every FO-formula $\varphi$ over $\sigma$, and every mapping $\nu$:
    \[
        (\mathfrak{A}, \nu) \models \varphi \ \iff \ (\mathfrak{B}, f \circ \nu) \models \varphi
    \]
    where $(f \circ \nu)$ is a mapping that associates $f(\nu(x))$ to each variable $x$.
    \item If $\varphi(x_1,\ldots,x_k)$ is an existential FO-formula over a vocabulary $\sigma$, $\mathfrak{A}$ is a $\sigma$-structure with domain $A$, $\mathfrak{A}'$ is an induced sub-structure of $\mathfrak{A}$ with domain $A'$, and $a_1,\ldots,a_k \in A'$:
    \[
    \mathfrak{A}' \models \varphi(a_1, \ldots, a_k) \ \implies \  \mathfrak{A} \models \varphi(a_1, \ldots, a_k)
    \]
\end{itemize}

All these ideas are standard, and can be found for example in the reference book of Chang and Keisler~\cite{chang1990model}.

We now proceed with the actual proof. For the sake of contradiction, assume that $\full(x)$ can be expressed in $\efoil$. More precisely, assume that $\varphi(x)$ is a formula in $\efoil$ such that for every $n \geq 1$, every model $\M$ of dimension $n$, and every partial instance $\es$ of dimension $n$:
\begin{eqnarray}\label{eq-n-full-efoil}
    \astruct_\M \models \varphi(\es) & \text{ if and only if} & \es \text{ is an instance}.
\end{eqnarray}
Let $\M_1$ be a model of dimension 1 such that $\M_1(\es) = 0$ for every instance $\es$. Then we have that:
\begin{eqnarray*}
\astruct_{\M_1} &=& \langle \{\bot,0,1\}, \pos^{\astruct_{\M_1}}, \subseteq^{\astruct_{\M_1}} \rangle,
\end{eqnarray*}
where $\pos^{\astruct_{\M_1}} = \emptyset$. Moreover, given condition \ref{eq-n-full-efoil}, we also know that $\astruct_{\M_1} \models \varphi((0))$. Let $\M_2$ be a model of dimension 2 such that $\M_2(\es) = 0$ for every instance $\es$. Then we have that:
\begin{eqnarray*}
\astruct_{\M_2} &=& \langle \{\bot,0,1\}^2, \pos^{\astruct_{\M_2}}, \subseteq^{\astruct_{\M_2}} \rangle,
\end{eqnarray*}
where $\pos^{\astruct_{\M_2}} = \emptyset$. Moreover, let $\astruct'$ be the sub-structure of $\astruct_{\M_2}$ induced by the set of instances $\{(\bot,\bot), (0,\bot), (1,\bot)\}$. Then we have that function 
$f : \{(\bot), (0), (1)\} \to \{(\bot,\bot), (0,\bot), (1,\bot)\}$ defined as $f((x)) = (x, \bot)$ is an isomorphism from $\astruct_{\M_1}$ to $\astruct'$ such that $f((0)) = (0,\bot)$. Hence, given that $\varphi(x)$ is a formula in first-order logic and $\astruct_{\M_1} \models \varphi((0))$, we conclude that $\astruct' \models \varphi((0,\bot))$. Moreover, given that $\astruct'$ is an induced sub-structure of $\astruct_{\M_2}$ and $\varphi(x)$ is an existential formula in first-order logic, we have that $\astruct_{\M_2} \models \varphi((0,\bot))$. Notice that this contradicts condition \ref{eq-n-full-efoil}, as $(0,\bot)$ is not an instance.

\section{Proof of Theorem \ref{theo:main-dt}}
\label{sec:evaluation}



Let us restate the theorem for the reader's convenience.
\begin{thmbis}{theo:main-dt}
There exists a formula $\psi(x)$ in $\foil$ for which \logicEvaluation$(\psi(x), \dt)$ and 
\logicEvaluation$(\psi(x), \obdd)$ are $\np$-hard. 
\end{thmbis}
\begin{proof}
We show that the problem is $\np$-hard by reducing from the
satisfiability problem for propositional formulas in 3-CNF. We will in fact show that hardness holds already for the class $\dt \cap \obdd$, which proves both cases a once. Let
$\varphi = C_1 \wedge \cdots \wedge C_n$ be a propositional formula,
where each $C_i$ is a disjunction of three literal and does
not contain repeated or complementary literals. Moreover, assume that
$\{x_1, \ldots, x_m\}$ is the set of variables occurring in $\varphi$, and the proof will use partial instances of dimension $n+m$. Notice that the last $m$ features of such a partial instance $\es$ naturally define a truth assignment for the propositional formula $\varphi$. More precisely, for every $i \in \{1, \ldots, n\}$, we use notation $\es(C_i) = 1$ to indicate that there is a disjunct $\ell$ of $C_i$ such that $\ell = x_j$ and $\es[n+j] = 1$, or $\ell = \neg x_j$ and $\es[n+j] = 0$, for
some $j \in \{1,\ldots,m\}$. Furthermore, we say $\es(\varphi) = 1$ if $\es(C_i) = 1$ for every $i \in \{1, \ldots, n\}$.

We will build an \emph{ordered decision tree} (thus belonging to $\dt \cap \obdd)$, over the natural ordering $1 < 2 < \cdots < n+m-1 < n+m$. Let us denote this ordering with ${\prec}$ in order to avoid confusion.  For each clause $C_i$ ($i \in \{1, \ldots, n\}$), let  $\T_{C_i}$ be
a decision tree of dimension $n+m$ (but that will only use features $n+1, \ldots, n+m$) such that for every entity $\es$:
$\T_{C_i}(\es) = 1$ if and only $\es(C_i) = 1$. Moreover, we require each $\T_{C_i}$ to be ordered with respect to ${\prec}$, Notice that  $\T_{C_i}$
can be constructed in constant time as it only needs to contain at most
eight paths of depth 3. For example, assuming that $C = (x_1 \vee x_2 \vee x_3)$,
a possible decision tree $\T_C$ is depicted in the following figure:
\begin{center}
\begin{tikzpicture}[every node/.style={font=\footnotesize, scale=0.8}]
\node[circle,draw=black] (n) {$n+1$};
\node[circle,draw=black,below left = 5mm and 28mm of n] (n0) {$n+2$};
\node[circle,draw=black,below right = 5mm and 28mm of n] (n1) {$n+2$};
\node[circle,draw=black,below left = 5mm and 10mm of n0] (n00) {$n+3$};
\node[circle,draw=black,below right = 5mm and 10mm of n0] (n01) {$n+3$};
\node[circle,draw=black,below left = 5mm and 10mm of n1] (n10) {$n+3$};
\node[circle,draw=black,below right = 5mm and 10mm of n1] (n11) {$n+3$};
\node[circle,draw=black,below left = 5mm and 5mm of n00] (n000) {$\false$};
\node[circle,draw=black,below right = 5mm and 5mm of n00] (n001) {$\true$};
\node[circle,draw=black,below left = 5mm and 5mm of n01] (n010) {$\true$};
\node[circle,draw=black,below right = 5mm and 5mm of n01] (n011) {$\true$};
\node[circle,draw=black,below left = 5mm and 5mm of n10] (n100) {$\true$};
\node[circle,draw=black,below right = 5mm and 5mm of n10] (n101) {$\true$};
\node[circle,draw=black,below left = 5mm and 5mm of n11] (n110) {$\true$};
\node[circle,draw=black,below right = 5mm and 5mm of n11] (n111) {$\true$};

\path[arrout] (n) edge node[above] {$0$} (n0);
\path[arrout] (n) edge node[above] {$1$} (n1);
\path[arrout] (n0) edge node[above] {$0$} (n00);
\path[arrout] (n0) edge node[above] {$1$} (n01);
\path[arrout] (n1) edge node[above] {$0$} (n10);
\path[arrout] (n1) edge node[above] {$1$} (n11);
\path[arrout] (n00) edge node[above] {$0$} (n000);
\path[arrout] (n00) edge node[above] {$1$} (n001);
\path[arrout] (n01) edge node[above] {$0$} (n010);
\path[arrout] (n01) edge node[above] {$1$} (n011);
\path[arrout] (n10) edge node[above] {$0$} (n100);
\path[arrout] (n10) edge node[above] {$1$} (n101);
\path[arrout] (n11) edge node[above] {$0$} (n110);
\path[arrout] (n11) edge node[above] {$1$} (n111);
\end{tikzpicture}
\end{center} 
Moreover, define $\T_\varphi$ as the following decision tree, clearly ordered with respect to ${\prec}$:
\begin{center}
\begin{tikzpicture} 
\node[circle,draw=black] (c1) {$1$};
\node[below left = 6mm and 6mm of c1] (tc1) {$\T_{C_1}$};
\node[circle,draw=black,below right = 6mm and 6mm of c1] (c2) {$2$};
\node[below left = 6mm and 6mm of c2] (tc2) {$\T_{C_2}$};
\node[circle,draw=black,below right = 6mm and 6mm of c2] (c3) {$3$};
\node[below left = 6mm and 6mm of c3] (tc3) {$\T_{C_3}$};
\node[below right = 6mm and 6mm of c3] (d) {$\cdots$};
\node[circle,draw=black,below right = 6mm and 6mm of d] (cn) {$n$};
\node[below left = 6mm and 6mm of cn] (tcn) {$\T_{C_n}$};
\node[circle,draw=black,below right = 6mm and 6mm of cn, minimum size=8mm] (o) {$\true$};

\path[arrout] (c1) edge node[above] {$0$} (tc1);
\path[arrout] (c1) edge node[above] {$1$} (c2);
\path[arrout] (c2) edge node[above] {$0$} (tc2);
\path[arrout] (c2) edge node[above] {$1$} (c3);
\path[arrout] (c3) edge node[above] {$0$} (tc3);
\path[arrout] (c3) edge node[above] {$1$} (d);
\path[arrout] (d) edge node[above] {$1$} (cn);
\path[arrout] (cn) edge node[above] {$0$} (tcn);
\path[arrout] (cn) edge node[above] {$1$} (o);
\end{tikzpicture}
\end{center} 

Finally, define $\es$ as a partial instance of dimension $n+m$ such that $\es[i] = 1$ for every $i \in \{1,\ldots,n\}$, and $\es[n+j] = \bot$ for every $j \in \{1, \ldots, m\}$, and define $\psi(x)$ as the following formula in \foil~(equivalent to the formula presented in the body of the paper):
\begin{multline}\label{eq-def-fo2}
\psi(x) \ =\ \exists y \, (x \subseteq y \wedge \full(y) \wedge
\forall z \, ((z \subseteq y \wedge \neg y \subseteq z)\ \rightarrow\\
\exists u \, (z \subseteq u \wedge \neg u \subseteq z \wedge \neg \full(u)) \vee
\forall v \, ((z \subseteq v \wedge \neg v \subseteq z) \rightarrow \pos(v)))).
\end{multline}
Interestingly, $\psi(x)$ can be rewritten by using only two variables, which proves that an even more restricted fragment of $\foil$ is hard to evaluate. The following two-variable formula is equivalent to $\psi(x)$:
\begin{multline*}
\exists y \, (x \subseteq y \wedge \full(y) \wedge
\forall x \, ((x \subseteq y \wedge \neg y \subseteq x)\ \rightarrow\\
\exists y \, (x \subseteq y \wedge \neg y \subseteq x \wedge \neg \full(y)) \vee
\forall y \, ((x \subseteq y \wedge \neg y \subseteq x) \rightarrow \pos(y)))).
\end{multline*}

In what follows, we prove that $\varphi$ is
satisfiable if and only if $\astruct_{\T_\varphi} \models \psi(\es)$,
from which we conclude that the theorem holds.
\begin{itemize}
\item[$(\Leftarrow)$] Assume that $\astruct_{\T_\varphi} \models \psi(\es)$, and assume that $\es_1$ is a witness for the variable $y$, that is, 
\begin{multline*}
\astruct_{\T_\varphi} \ \models \ \es \subseteq \es_1 \wedge \full(\es_1) \wedge
\forall z \, ((z \subseteq \es_1 \wedge \neg \es_1 \subseteq z) \ \rightarrow\\
\exists u \, (z \subseteq u \wedge \neg u \subseteq z \wedge \neg \full(u)) \vee
\forall v \, ((z \subseteq v \wedge \neg v \subseteq z) \rightarrow \pos(v))).
\end{multline*}
In what follows, we show that $\es_1(\varphi) =1$, from which we
conclude that $\varphi$ is satisfiable. Fix an arbitrary $i \in \{1, \ldots, n\}$.
Then, let $\es_2$ be a partial instance of dimension $n+m$ such that (i) $\es_2[i] = \bot$; (ii)
$\es_2[j] = 1$ for each $j \in \{1, \ldots, n\}$ with $j \neq i$, and (iii) $\es_2[n+k] = \es_1[n+k]$ for each $k \in \{1, \ldots, m\}$. Then given that $(\es_2 \subseteq \es_1 \wedge \neg \es_1 \subseteq \es_2)$, we have that:
\begin{multline*}
\astruct_{\T_\varphi} \ \models \ 
\exists u \, (\es_2 \subseteq u \wedge \neg u \subseteq \es_2 \wedge \neg \full(u)) \ \vee \\
\forall v \, ((\es_2 \subseteq v \wedge \neg v \subseteq \es_2) \rightarrow \pos(v)).
\end{multline*}
Therefore, given that $\es_2$ assigns value $\bot$ to exactly one feature, we conclude that:
\begin{eqnarray}\label{eq-es2}
\astruct_{\T_\varphi} & \models &
\forall v \, ((\es_2 \subseteq v \wedge \neg v \subseteq \es_2) \rightarrow \pos(v)).
\end{eqnarray}
Define $\es_3$ as an instance of dimension $n+m$ such that $\es_3[i] = 0$,
$\es_3[j] = 1$ for each $j \in \{1, \ldots, n\}$ with $j \neq i$,
and $\es_3[n+k] = \es_2[n+k]$ for each $k \in \{1, \ldots, m\}$.
Then, by considering that
$(\es_2 \subseteq \es_3 \wedge \neg \es_3 \subseteq \es_2)$ holds, we
conclude from \ref{eq-es2} that $\pos(\es_3)$ holds. Therefore, given that $\es_3[i] = 0$ and $\es_3[j] = 1$ for each $j \in \{1, \ldots, n\}$ with $j < i$, we have that $\T_{C_i}(\es_3) = 1$, from which we deduce that $\T_{C_i}(\es_1) = 1$, since $\es_1[n+j] = \es_2[n+j] = \es_3[n+j]$ for every $j \in \{1, \ldots, m\}$. As $i$ is an arbitrary element in the set $\{1, \ldots, n\}$, we conclude that $\es_1(C_i) = 1$ for every $i \in \{1, \ldots, n\}$ and, thus,
$\es_1(\varphi) = 1$, which was to be shown.

\item[$(\Rightarrow)$] Assume that $\varphi$ is satisfiable, and let $\sigma$ be a truth assignment such that $\sigma(\varphi) = 1$. Moreover, define an instance $\es_1$ of dimension $n+m$ such that $\es_1[i] = 1$ for each $i \in \{1, \ldots, n\}$ and $\es_1[n+j] = \sigma(x_j)$ for each $j \in \{1, \ldots, m\}$. Then we have that $\es_1(\varphi) = 1$, $\es \subseteq \es_1$ and $\full(\es_1)$ hold. Next we show that:
\begin{multline*}
\astruct_{\T_\varphi} \ \models \ \forall z \, ((z \subseteq \es_1 \wedge \neg \es_1 \subseteq z) \ \rightarrow\\
\exists u \, (z \subseteq u \wedge \neg u \subseteq z \wedge \neg \full(u)) \vee
\forall v \, ((z \subseteq v \wedge \neg v \subseteq z) \rightarrow \pos(v))),
\end{multline*}
from which we conclude that $\astruct_{\T_\varphi} \models \psi(\es)$.
Let $\es_2$ be a partial instance of dimension $n+m$ such that
$(\es_2 \subseteq \es_1 \wedge \neg \es_1 \subseteq \es_2)$ holds. We
need to prove that:
\begin{multline*}
\astruct_{\T_\varphi} \ \models \ \exists u \, (\es_2 \subseteq u \wedge \neg u \subseteq \es_2 \wedge \neg \full(u)) \ \vee\\
\forall v \, ((\es_2 \subseteq v \wedge \neg v \subseteq \es_2) \rightarrow \pos(v)).
\end{multline*}
Notice that $\es_2$ assigns value $\bot$ to at least one feature in $X$,
since $(\es_2 \subseteq \es_1 \wedge \neg \es_1 \subseteq \es_2)$
holds. If $\es_2$ assigns value $\bot$ to at least two features, then
clearly $\astruct_{\T_\varphi} \models \exists u \, (\es_2 \subseteq
u \wedge \neg u \subseteq \es_2 \wedge \neg \full(u))$. Hence, assume
that $\es_2$ assigns value $\bot$ to exactly one feature, and consider
the following cases.
\begin{itemize}
\item
If $\es_2[n+j] = \bot$ for some $j \in \{1, \ldots, m\}$. Then for
every partial instance $\es_3$ of dimension $n+m$ such that
$(\es_2 \subseteq \es_3 \wedge \neg \es_3 \subseteq \es_2)$ holds, we
have that $\es_3[i] = 1$ for every $i \in \{1, \ldots,
n\}$. Therefore, from the definition of $\T_\varphi$, we conclude that
$\pos(\es_3)$ holds. Thus, we have that $\astruct_{\T_\varphi} \models \forall
v \, ((\es_2 \subseteq v \wedge \neg
v \subseteq \es_2) \rightarrow \pos(v))$.

\item
If $\es_2[i] = \bot$ for some $i \in \{1, \ldots, n\}$. Then assume
that $\es_3$ is a partial instance of dimension $n+m$ such that
$(\es_2 \subseteq \es_3 \wedge \neg \es_3 \subseteq \es_2)$ holds. If
$\es_3[i] = 1$, then we have that $\es_3[j] = 1$ for every
$j \in \{1, \ldots, n\}$, and we conclude by definition of
$\T_\varphi$ that $\pos(\es_3)$ holds. If $\es_3[i] = 0$, then we
conclude that $\T_\varphi(\es_3) = \T_{C_i}(\es_3)$, since $\es_3[j]
= 1$ for every $j \in \{1, \ldots, n\}$ such that $j < i$. Given that $\es_1$, $\es_2$ and $\es_3$ only differ in the value of $f_i$, we have that $\es_3[n+k] = \es_2[n+k]
= \es_1[n+k]$ for every $k \in \{1, \ldots, m\}$, so that
$\T_{C_i}(\es_3) = \T_{C_i}(\es_1)$. But then given that
$\es_1(\varphi) = 1$, we know that $\T_{C_i}(\es_1) = 1$, which
implies that $\T_\varphi(\es_3) = \T_{C_i}(\es_3) = \T_{C_i}(\es_1) =
1$. We conclude again that $\pos(\es_3)$ holds, from which we deduce
that $\astruct_{\T_\varphi} \models \forall v \, ((\es_2 \subseteq
v \wedge \neg v \subseteq \es_2) \rightarrow \pos(v))$.
\end{itemize}
This concludes the proof of the theorem.
\end{itemize}

\end{proof}

\section{Proof of Proposition~\ref{prop:dt-obdd-efoil}}
\label{sec:tractable-efoil}
Let us restate and prove the corresponding proposition.

\begin{propbis}{prop:dt-obdd-efoil}
Let $\varphi$ be a query in $\efoil$ or $\afoil$. Then $\logicEvaluation(\varphi, \dt)$ and $\logicEvaluation(\varphi, \obdd)$ can be solved in polynomial time.
\end{propbis}

\begin{proof}
We will prove this for the more general class of $\fbdd$, that contains both $\dt$ and $\obdd$.
Assume that the input formula is of the form $\varphi = \exists x_1, \cdots, \exists x_k \psi(x_1, \ldots, x_k)$, with $\psi$ quantifier-free, and let $\M$ be the input model with $\dim(\M) = n$. Our algorithm will try to construct a valuation $\es_1, \ldots, \es_k$ of the variables of $\varphi$ such that $\M \models \psi(\es_1, \ldots, \es_k)$, and if this fails, it will be certain that no satisfying valuation exists. 

We assume as well the input formula has constants but no free-variables as if the input instance has free variable we can simply replace them by the partial instances $\es_i$ supplied in the input.  Let $V = \{\vx_1, \ldots, \vx_k\}$ be the variables mentioned in $\varphi$, of which there is only a constant number as $\varphi$ is fixed. Let $E = \{\es_1, \ldots, \es_k\}$ be their corresponding undetermined instances.

For each element in the domain of $\mathfrak{A}_\M$, that is, for each tuple in $\{0, 1, \bot\}^n$, we define its \emph{type} as the set of unary predicates of $\efoil$ that it satisfies when interpreted over $\mathfrak{A}_\M$. In the case of $\efoil$, this set corresponds either to $\{\pos\}$ or to $\{\neg \pos\}$, but we present the general strategy as it can be used for bigger fragments of $\foil$, as shown later in the proof of Theorem~\ref{theo:pap-pan}.

Let $\T$ be the set of types which is of course a fixed set independent of $\M$. We will guess the type of each instance $\es \in E$. More formally, we can iterate over all type assignments $\tau : E \to \T$ as there is only fixed number of them. Similarly, we can define a \emph{containment assignment} $\gamma$ as an assignment of all the ordered pairs $(\es_i, \es_j)$ to $\{0, 1\}$, with the meaning that $\es_i \subseteq \es_j$ iff $\gamma(\es_i, \es_j) = 1$. Such an assignment is said to be \emph{possible} only if it holds the properties of a partial order. Given a possible containment assignment $\gamma$, we can interpret it as a pair of sets
\[
P = \{ (\es_i, \es_j) \mid  \gamma(\es_i, \es_j) = 1 \} \quad ; \quad N = \{ (\es_i, \es_j) \mid  \gamma(\es_i, \es_j) = 0 \} 
\]

Note as well that there is a constant number of possibilities for the pair $P, N$. Because the formula $\varphi$ is existential, if there is an determinization of $E$ that models $\varphi$,  then there is a pair $(\tau, \gamma = (P, N))$ where $\tau$ is a possible type assignment and $\gamma = (P, N)$ is a possible containment assignment, such that $E$ is \emph{consistent} with both $\tau$ and $\gamma$. More precisely, $E$ is consistent with $\tau$ and $\gamma$ iff:
\begin{itemize}
    \item For every $\es \in E$ and every unary predicate $\rho$, 
    \[
        \rho \in \tau(\es) \iff \es \in \rho^{\mathfrak{A}_\M}
    \]
    \item For every pair $\es_i, \es_j \in E$, 
    \[
        \gamma(\es_i, \es_j) = 1 \iff (\es_i, \es_j) \in  \subseteq^{\mathfrak{A}_\M}
    \]
\end{itemize}

We can afford to iterate over the constantly many pairs $(\tau, \gamma)$, and for each pair $(\tau, \gamma)$ it is trivial to decide whether $\varphi$ gets satisfied under said assignments (simply by replacing every atomic term in $\varphi$ by the value assigned to it by $\tau$ or $\gamma$). Therefore,
in order to prove the whole theorem, it is enough to design a polynomial time algorithm that decides whether there determinization of $E$ that is a consistent with a given pair $(\tau, \gamma)$. More precisely, proving the next claim will be enough to conclude our proof.

\begin{claim}
Given a pair $(\tau, \gamma = (P, N))$, one can check in polynomial time whether there is a determinization of $E$ that is consistent with $(\tau, \gamma)$.
\end{claim}
\begin{proof}[Proof of Claim 1]
First, as the desired determinization $E$ must be consistent with $N$, it must hold that for every fact $(\es_i, \es_j) \in N$, there is an index $1 \leq k \leq n$ such that $\es_i[k] \neq \bot$ and $\es_i[k] \neq \es_j[k]$. We can afford to guess, for each of the constantly many facts $(\es_i, \es_j) \in N$, an index $k$ and the values of $\es_i[k], \es_j[k]$, that certify the fact.  After said guesses have been made, we can assume a set $F$ of guessed facts of the form $\es[k] = \alpha$, with $\alpha \in \{0,1, \bot\}$. Then, for every fact in $F$ of the form $\es[k] = \beta$, with $\beta \in \{0, 1\}$, we include in $F$ all facts of the form $\es'[k] = \beta$ for every $\es'$ such that $(\es, \es') \in P$. Also, for every fact in $F$ of the form $\es[k] = \bot$, we include $\es'[k] = \bot$ for every $\es'$ such that $(\es', \es) \in P$.  As any determinization of $E$ respecting $F$ will at least be consistent with $N$, it remains only to check whether there is an interpretation $E$ respecting to $F$ that is consistent with $\tau$ and $P$. 

If $F$ fully determines some predicates that a certain instance $\es \in E$ must satisfy, for example because $F$ contains facts $\es[k] = \beta \in \{0, 1\}$ for all $1 \leq k \leq n$ and thus we know that $\es$ must be a full instance, we can check whether $\pos(\es)$ and if that holds reject immediately if $\pos(\es) \not\in \tau(\es)$. Therefore, we can safely assume this is not the case, and that $\tau$ is not directly contradicted by $F$. We thus modify the undetermined instances $\es_1, \ldots, \es_k$ according to $F$. Let us now interpret $P$ as a directed acyclic graph $G$ obtained in the following way: (i) create a node for every instance $\es \in E$, (ii) create an edge $\es \to \es'$ iff $(\es, \es') \in P$, (iii) collapse strongly connected components to a single node. Note that, as strongly connected components before the last step correspond to instances that must be equal, we can think of them as a single instance, because forcefully $\tau$ must assign the same to each of them. We can now view our problem as that of determinizing every node in a DAG $G$, in such a way that the containment dictated by the graph is satisfied, and so is $\tau$.

If $G$ has multiple connected components (it will only have a constant number of them), it is easy to see that we can simply  make the check for each of them separately, and return that the instance is positive if every connected component holds the check. This is because different connected components do not share instances $\es$, and thus a determinization of a connected component is always compatible with the determinization of another connected component. As a consequence, our problem is now even smaller; we need to show that it is possible to determine in polynomial time if the undetermined components in each node of a given connected DAG $G$ can be assigned values that are consistent with given assignments $\tau$ and $P$, assuming the guessed facts $F$. 

We now show a direct simple algorithm for this problem:
\begin{enumerate}
    \item Choose an arbitrary topological ordering $\phi$ of $G$.
    \item Iterate over the nodes according to $\phi$, and for each node $\es$ do the next step.
    \item  If $\pos \in \tau(\es)$, go to step 4., otherwise go to 5.
    \item  We determinize $\es$ in an arbitrary way that is accepted by $\M$. This is easily done in polynomial time for FBDDs; it is enough to prune the edges of the FBDD that contradict a defined feature in $\es$, and then find any positive leaf of the resulting model.
    Take $\es$ to be the next node according to $\phi$ and go back to 3. If there is no next node, go to 6.
    \item Assign every undetermined component of $\es$ to $\bot$, as that does not restrict any future choices while ensuring that $\pos \not \in \tau(\es)$.  Take $\es$ to be the next node according to $\phi$ and go back to 3.  If there is no next node, go to 6.
    \item Now that nodes have no undetermined components, check that every fact dictated by $\tau$ is true for the values that have been determined.  If all the facts are correctly satisfied, return Yes, otherwise return No.
\end{enumerate}

It is clear that, if the preceding algorithm returns Yes, then it is correct, as it has a concrete determinization consistent with $\tau$, and it must be consistent with $P$ as every undetermined component that is assigned $0$ or $1$ its propagated to the successors in the graph. It only remains to justify that it is correct when it returns No. 
Assume, looking for a contradiction, that the algorithm returns No but there actually exists a determinization $B$ of $E$  that is consistent with $\tau$ and $P$, assuming the guessed facts $F$. Let $A$ be the determinization that the algorithm tested in step 6, and let $i$ be the first node according to $\phi$, the choice of the algorithm in step 1, such that $A(\phi_i) \neq B(\phi_i)$. Such an index must exists because $A$ must differ from $B$. Among all determinizations that are consistent with $\tau$ and $P$, let $B'$ be the one that maximizes the index $i$ of its first difference with $A$. Then, let $\es$ be $i$-th node according to $\phi$, and thus the first node where $A$ and $B'$ differ.  If $\pos \in \tau(\es)$,
the algorithm determinized $\es$ in an arbitrary way that makes $\es$ a positive instance.  But then, as $\es$ is positive (and therefore a full instance), it cannot have any successors in $G$, and thus if we let $B'' \coloneqq B'$ except for $B''(\es) \coloneqq A(\es)$, then $B''$ must also be consistent with $\tau$, which contradicts the maximality of $i$. If $\pos \not \in \tau(\es)$ we have two cases, either $B'(\es)$ is a full instance or not. If it is, then again it has no successors in $G$, so it must be that the inconsistency is that the algorithm determinized $\es$ in a way that makes it a positive instance. This is clearly not possible, as the only step in the algorithm that introduces values different from $\bot$, and thus that makes feasible for $\es$ to be a positive instance, is step 4, which occurs exactly when $\pos \in \tau(\es)$. It remains to see the case where $B'(\es)$ is not a full instance. Assume $j$ is the first component for which $A(\es)[j] \neq B'(\es)[j]$. 
If $A(\es)[j] = \bot$, then note that every successor of $B'(\es)$ is also a successor of $A(\es)$ and thus if $B'(\es)$ is consistent with $\tau$, then so is $A(\es)$. This implies the inconsistency in $\tau$ must appear later in $\phi$, and thus we can again take $B''$ equal to $B'$ except for $B''(\es) \coloneqq A(\es)$ which will contradict the maximality of $i$. If $A(\es)[j] \neq \bot$, then said value need to come from $F$, as the algorithm only introduces the value $\bot$ for instances where $\pos \not \in \tau(\es)$, which means that $B'(\es)[j] = A(\es)[j]$, as $B'$ must also respect $F$, which contradicts the minimality of $j$.

\end{proof}

As the preceding claim has been proved, and there are constantly many pairs $(\tau, \gamma)$ to consider, there is a polynomial time algorithm for the whole problem.
\end{proof}

\section{Proof of Theorem \ref{theo:pap-pan}, Proposition~\ref{prop:pap-pan-ptron} and 
Proposition~\ref{prop:hardness-dt-ef}}
\label{sec:tractable-efoilp}
Before the proofs, let us gain a better understanding on the $\pap$ and $\pan$ formulas. Recall that

\begin{multline}
\label{eq:pap}
\pap(x,y,z) \ = \
\exists u \, [x \subseteq u \wedge \allpos(u) \ \wedge\\
\exists v \, (y \subseteq v \wedge u \subseteq  v) \wedge \exists w \, (z \subseteq w \wedge u \subseteq w)]
\end{multline}

    \newcommand{\unknown}{\diamondsuit}

We will prove that this query captures an important computational problem. Let us introduce a fourth kind of value: $\unknown$, so we now define \emph{undetermined instances} as tuples in $\{0, 1, \bot, \unknown\}^n$ for some $n \geq 1$. A component with value $\unknown$ is said to be \emph{undetermined}. Given an undetermined instance $\es$ of dimension $n$, we say that a partial instance $\es'$ of dimension $n$ is a \emph{determinization} of $\es$ if for $\es$ matches $\es'$ in every component that is not undetermined. Note that $\es$ cannot have undetermined components as it is a partial instance (i.e., a tuple in $\{0,1,\bot\}^n$).

\newcommand{\dap}{\textsc{DeterminizationAllPos}}

\newcommand{\dan}{\textsc{DeterminizationAllNeg}}

Consider now the following computational problem:
\begin{center}
\fbox{\begin{tabular}{rl}
Problem: & \dap$(\C)$\\
Input: & A model $\M \in \C$ of dimension $n$, and an undetermined instance $\es$ of dimension $n$\\
Output: & \textsc{Yes}, if there is a determinization $\es'$ of $\es$ such that all completions of $\es'$ are positive,\\
& and \textsc{No} otherwise
\end{tabular}}
\end{center}

It turns out that $\dap$ is intimately related to $\pap$:

\begin{lemma}
\label{lemma:pap-iff-dap}
Let $\C$ be any class of models. Then $\logicEvaluation(\pap,\C)$ can be solved in polynomial time if and only if $\dap(\C)$ can also be solved in polynomial time.
\end{lemma}

\begin{proof}
We prove both directions as separate claims for an arbitrary class of models $\C$.

\begin{claim}
If  $\logicEvaluation(\pap,\C)$ can be solved in polynomial time then $\dap(\C)$ can also be solved in polynomial time.
\end{claim}
\begin{proof}
Assume that $\logicEvaluation(\pap,\C)$ can be solved in polynomial time for $\C$. Then, consider an instance $(\M, \es)$  of $\dap$, and let $n$ be the dimension of said instance. From $\es$, we build three partial instances $\es_x, \es_y, \es_z$ in the following way:
\begin{itemize}
    \item $\es_x$ is a determinization of $\es$ such that every undetermined component of $\es$ is replaced by $\bot$ in $\es_x$.
    \item $\es_y$ is a partial instance that has a $1$ in every component where $\es$ has $\bot$, and $\bot$ in every other component.
    
     \item $\es_z$ is a partial instance that has a $0$ in every component where $\es$ has $\bot$, and $\bot$ in every other component.
    
\end{itemize}

We now claim that  $\M \models \pap(\es_x, \es_y, \es_z)$ if and only if $(\M, \es)$ is a positive instance of $\dap$. 

Indeed, assume first that $\M \models \pap(\es_x, \es_y, \es_z)$, and let be $\es_u, \es_v, \es_w$ be their witnesses. Trivially, $\es_u$ is a partial instance for which every completion is positive. Note that because $\M \models \es_x \subseteq \es_u$ and the definition of $\es_x$, we have that the defined components of $\es_u$ and $\es$ match. It only remains to see that if $\es[i] = \bot$ for some $1 \leq i \leq n$, then $\es_u[i] = \bot$ as well. Assume to the contrary that for some $i$ it happens that $\es[i] = \bot$ but $\es_u[i] \neq \bot$. If $\es_u[i] = 0$, then $\es_v[i] = 0$, as $\M \models \es_u \subseteq \es_v$. But $\es_v[i] = 0$ contradicts the fact that $\es_y[i] = 1$ (by construction) as $\M \models \es_y \subseteq \es_v$. Similarly, if $\es_u[i] = 1$, then $\es_w[i] = 1$, as $\M \models \es_u \subseteq \es_w$. But $\es_w[i] = 1$ contradicts the fact that $\es_z[i] = 0$ (by construction) as $\M \models \es_z \subseteq \es_w$.

For the other direction, assume $(\M, \es)$ is a positive instance of $\dap$, and let $\es'$ be the determinization of $\es$ that serves as a witness. We claim that $\es_u \coloneqq \es'$ is a witness for $\M \models \pap(\es_x, \es_y, \es_z)$. Indeed, it is trivial that $\M \models \es_x \subseteq \es_u$ as both $\es_x$ and $\es_u$ are determinization of $\es$, but $\es_x$ replaced undetermined components by $\bot$. It is also clear that $\M \models \allpos(\es_u)$, as all completions of $\es'$ are positive by definition. Then, let $\es_v$ be the completion of $\es_u$ that replaces every $\bot$ component of $\es_u$ with $1$. Let $\es_w$ be defined analogously but replacing $\bot$ with $0$. It is then easy to check that 
\[
\M \models  (\es_y \subseteq \es_v \wedge \es_u \subseteq \es_v) \wedge  (\es_z \subseteq \es_w \wedge \es_u \subseteq \es_w)
\]
and thus $\M \models \pap(\es_x, \es_y, \es_z)$, which is enough to conclude the proof.
\end{proof}

\begin{claim}
If $\dap(\C)$  can be solved in polynomial time then $\logicEvaluation(\pap,\C)$  can also be solved in polynomial time.
\end{claim}
\begin{proof}
Assume that $\dap(\C)$ can be solved in polynomial time. Then, let $(\M, \es_x, \es_y, \es_z)$ be an input of $\logicEvaluation(\pap,\C)$, and let $n = \dim(\M)$. 

First, we claim that if for some $1 \leq i \leq n$ it happens that $\es_y[i] \neq \bot \neq \es_x[i]$ or $\es_z[i] \neq \bot \neq \es_x[i]$, then we can trivially deduce that $(\M, \es_x, \es_y, \es_z)$ is a negative instance of $\logicEvaluation(\pap,\C)$. Indeed, if $\es_y[i] \neq \bot$, and $(\M, \es_x, \es_y, \es_z)$ were to be a positive instance, then there would exists witnesses $\es_u, \es_v, \es_w$, which would hold the following properties:
\begin{enumerate}
    \item $\es_u[i] = \es_x[i]$, as $\M \models \es_x \subseteq \es_u$ and $\es_x[i]$ is assumed to not be $\bot$.
    \item $\es_v[i] = \es_y[i]$, as $\M \models \es_y \subseteq \es_v$ if $\es_y[i] \neq \bot$.
    \item $\es_w[i] = \es_z[i]$, as $\M \models \es_z\subseteq \es_w$
   if $\es_z[i] \neq \bot$.
    \item $\es_v[i] = \es_u[i]$, as $\M \models \es_u \subseteq \es_v$ and $\es_u[i] = \es_x[i]$ is assumed to not be $\bot$.
     \item $\es_w[i] = \es_u[i]$, as $\M \models \es_u \subseteq \es_w$ and $\es_u[i] = \es_x[i]$ is assumed to not be $\bot$.
\end{enumerate}

Transitively, it would follow if $\es_y[i] \neq \bot$, then $\es_x[i] = \es_y[i]$, and if $\es_z[i] \neq \bot$, then $\es_x[i] = \es_z[i]$, which contradicts the assumption.

Therefore, we can safely assume from now on that, if $\es_x[i] \neq \bot$, then either $\es_y[i] = \bot$ or $\es_y[i] = \es_x[i]$, and the same holds for $\es_z[i]$. We now define $\es$ as an undetermined instance that is equal to $\es_x$ except that it has $\unknown$ in every component where $\es_x$ has $\bot$. We now claim that $(\M, \es_x, \es_y, \es_z)$ is a positive instance of $\logicEvaluation(\pap,\C)$ if and only if $(\M, \es)$ is a positive instance of $\dap$. 

Indeed, assume that $(\M, \es_x, \es_y, \es_z)$ is a positive instance of $\logicEvaluation(\pap,\C)$. Then, it is trivial that its witness $\es_u$ is a determinization of $\es$ with only positive completions. For the other direction, if $(\M, \es)$ is a positive instance of $\dap$ with witness $\es'$, then it is easy to see that taking $\es_u \coloneqq \es_v \coloneqq \es_w \coloneqq \es'$ proves that  $(\M, \es_x, \es_y, \es_z)$ is a positive instance of $\logicEvaluation(\pap,\C)$, as clearly
\(
\M \models \es_x \subseteq \es_u \land \allpos(\es_u)
\)
and also trivially
\(
\M \models \es_u \subseteq \es_v \land \es_u \subseteq \es_w
\), thus
leaving only 
\(\M \models \es_y \subseteq \es_v \land \es_z \subseteq \es_w
\)
to justify, which we do simply by using the previous fact that if $\es_y[i] \neq \bot$ for some $i$, $\es_x[i] = \es_y[i]$, from which we know that $\es_v[i] = \es_y[i]$,as $\M \models \es_x \subseteq \es_u \subseteq \es_v$. The same reasoning justifies that $\M \models \es_z \subseteq \es_w$

\end{proof}

The lemma follows directly from the combination of both claims.
\end{proof}

It is easy to see that the same proof applies to $\pan$ and $\textsc{DeterminizationAllNeg}$. We now restate the main theorem of this section and proceed to prove it.

\begin{thmbis}{theo:pap-pan}
For every class $\C$ of models, the following conditions are equivalent:
(a) $\logicEvaluation(\varphi ,\C)$ can be solved in polynomial time for each query $\varphi$ in $\efoilp$;
%
(b) $\logicEvaluation(\pap,\C)$ and $\logicEvaluation(\pan,\C)$ can be solved 
in polynomial~time.
\end{thmbis}

\begin{proof}
The fact that $(a)$ implies $(b)$ is trivial as $\pap$ and $\pan$ can be written in $\efoilp$ as shown in the body of the paper. It remains to prove that $(b)$ implies $(a)$.  The proof is an extension of the proof of Proposition~\ref{prop:dt-obdd-efoil}. As it is again constructive and  technical, let us first present a sketch. We assume unary predicates $\exneg, \expos$, that trivially allow for expressing $\allpos$ and $\allneg$.

\paragraph{Sketch of proof} Assume that the input formula is of the form $\varphi = \exists x_1, \cdots, \exists x_k \psi(x_1, \ldots, x_k)$, with $\psi$ quantifier-free, and let $\M$ be the input model with $\dim(\M) = n$. Our algorithm will try to construct a valuation $\es_1, \ldots, \es_k$ of the variables of $\varphi$ such that $\M \models \psi(\es_1, \ldots, \es_k)$, and if this fails, it will be certain that no satisfying valuation exists. 
In order to do so, the algorithm starts taking $\es_1, \ldots, \es_k$ as undetermined instances, and in particular it starts setting $\es_1 = \es_2 = \cdots = \es_k = \unknown^n$. Then, as $k$ is a fixed constant, the algorithm can afford to guess which unary predicates of $\efoilp$ will be satisfied by each $\es_i$, and also all the containments $\es_i \subseteq \es_j$ that hold.
Note that some of such guesses might be inconsistent, as for example, they could fail to respect the transitive property of ${\subseteq}$, or guess that an $\es_i$ will hold both $\pos$ and $\exneg$, which is not possible either.
Inconsistent guesses are simply discarded. As only constantly many guesses exists, the complicated part of the algorithm is: given a consistent guess, check if it is possible to determinize all instances $\es_1$ through $\es_k$ while respecting the guess. One can show that the complicated cases are captured by the $\dap$ and $\dan$ problems, which because of Lemma~\ref{lemma:pap-iff-dap} are solvable in polynomial time given condition $(b)$.

We assume as well the input formula has no free-variables, as it complicates the exposition without adding combinatorial insight.  Let $V = \{\vx_1, \ldots, \vx_k\}$ be the variables mentioned in $\varphi$, of which there is only a constant number as $\varphi$ is fixed. Let $E = \{\es_1, \ldots, \es_k\}$ be their corresponding undetermined instances, as the proof sketch suggests. 
Also, let $\M$ be the input model, and let $n = \dim(\M)$.
For each element in the domain of $\mathfrak{A}_\M$, that is, for each tuple in $\{0, 1, \bot\}^n$, we define its \emph{type} as the set of unary predicates of $\efoilp$ that it satisfies when interpreted over $\mathfrak{A}_\M$. Note that not all sets of unary predicates are possible types, as for example no tuple can satisfy the set $\{\pos, \exneg\}$. Let $\T$ be the set of types that are possible, which is of course a fixed set independent of $\M$. We will guess the type of each instance $\es \in E$. More formally, we can iterate over all type assignments $\tau : E \to \T$ as there is only fixed number of them. Similarly, we can define a \emph{containment assignment} $\gamma$ as an assignment of all the ordered pairs $(\es_i, \es_j)$ to $\{0, 1\}$, with the meaning that $\es_i \subseteq \es_j$ iff $\gamma(\es_i, \es_j) = 1$. Such an assignment is said to be \emph{possible} only if it holds the properties of a partial order. Given a possible containment assignment $\gamma$, we can interpret it as a pair of sets
\[
P = \{ (\es_i, \es_j) \mid  \gamma(\es_i, \es_j) = 1 \} \quad ; \quad N = \{ (\es_i, \es_j) \mid  \gamma(\es_i, \es_j) = 0 \} 
\]

Note as well that there is a constant number of possibilities for the pair $P, N$. Because the formula $\varphi$ is existential, if there is an determinization of $E$ that models $\varphi$,  then there is a pair $(\tau, \gamma = (P, N))$ where $\tau$ is a possible type assignment and $\gamma = (P, N)$ is a possible containment assignment, such that $E$ is \emph{consistent} with both $\tau$ and $\gamma$. More precisely, $E$ is consistent with $\tau$ and $\gamma$ iff:
\begin{itemize}
    \item For every $\es \in E$ and every unary predicate $\rho$, 
    \[
        \rho \in \tau(\es) \iff \es \in \rho^{\mathfrak{A}_\M}
    \]
    \item For every pair $\es_i, \es_j \in E$, 
    \[
        \gamma(\es_i, \es_j) = 1 \iff (\es_i, \es_j) \in  \subseteq^{\mathfrak{A}_\M}
    \]
\end{itemize}

We can afford to iterate over the constantly many pairs $(\tau, \gamma)$, and for each pair $(\tau, \gamma)$ it is trivial to decide whether $\varphi$ gets satisfied under said assignments (simply by replacing every atomic term in $\varphi$ by the value assigned to it by $\tau$ or $\gamma$). Therefore,
in order to prove the whole theorem, it is enough to design a polynomial time algorithm that decides whether there determinization of $E$ that is a consistent with a given pair $(\tau, \gamma)$. More precisely, proving the next claim will be enough to conclude our proof.

\begin{claim}
Given a pair $(\tau, \gamma = (P, N))$, one can check in polynomial time whether there is a determinization of $E$ that is consistent with $(\tau, \gamma)$.
\end{claim}

\begin{proof}
First, as the desired determinization $E$ must be consistent with $N$, it must hold that for every fact $(\es_i, \es_j) \in N$, there is an index $1 \leq k \leq n$ such that $\es_i[k] \neq \bot$ and $\es_i[k] \neq \es_j[k]$. We can afford to guess, for each of the constantly many facts $(\es_i, \es_j) \in N$, an index $k$ and the values of $\es_i[k], \es_j[k]$, that certify the fact. Also, for every  $\full \not \in \tau(\es)$, we can guess a component $\es[k] = \bot$. After said guesses have been made, we can assume a set $F$ of guessed facts of the form $\es[k] = \alpha$, with $\alpha \in \{0,1, \bot\}$. Then, for every fact in $F$ of the form $\es[k] = \beta$, with $\beta \in \{0, 1\}$, we include in $F$ all facts of the form $\es'[k] = \beta$ for every $\es'$ such that $(\es, \es') \in P$. Also, for every fact in $F$ of the form $\es[k] = \bot$, we include $\es'[k] = \bot$ for every $\es'$ such that $(\es', \es) \in P$.  As any determinization of $E$ respecting $F$ will at least be consistent with $N$, it remains only to check whether there is an interpretation $E$ respecting to $F$ that is consistent with $\tau$ and $P$. 

If $F$ fully determines some predicates that a certain instance $\es \in E$ must satisfy, for example because $F$ contains facts $\es[k] = \beta \in \{0, 1\}$ for all $1 \leq k \leq n$ and thus we know that $\es$ must be a full instance, we can reject immediately if $\full \not\in \tau(\es)$. Therefore, we can safely assume this is not the case, and that $\tau$ is not directly contradicted by $F$. Let us now interpret $P$ as a directed acyclic graph $G$ obtained in the following way: (i) create a node for every instance $\es \in E$, (ii) create an edge $\es \to \es'$ iff $(\es, \es') \in P$, (iii) collapse strongly connected components to a single node. Note that, as strongly connected components before the last step correspond to instances that must be equal, we can think of them as a single instance, because forcefully $\tau$ must assign the same to each of them. We can now view our problem as that of determinizing every node in a DAG $G$, in such a way that the containment dictated by the graph is satisfied, and so is $\tau$.

If $G$ has multiple connected components (it will only have a constant number of them), it is easy to see that we can simply  make the check for each of them separately, and return that the instance is positive if every connected component holds the check. This is because different connected components do not share instances $\es$, and thus a determinization of a connected component is always compatible with the determinization of another connected component. As a consequence, our problem is now even smaller; we need to show that it is possible to determine in polynomial time if the undetermined components in each node of a given connected DAG $G$ can be assigned values that are consistent with given assignments $\tau$ and $P$. 

We now show a direct simple algorithm for this problem:
\begin{enumerate}
    \item Choose an arbitrary topological ordering $\phi$ of $G$.
    \item Iterate over the nodes according to $\phi$, and for each node $\es$ do the next step.
    \item  If $\full \in \tau(\es)$, go to step 4., otherwise go to 5.
    \item We either have $\pos \in \tau(\es)$ or $\exneg \in \tau(\es)$, but not both. On the first case, solve the $\dap$ problem with input $(\M, \es)$ and determinize $\es$ accordingly. On the second case, solve the $\dan$ problem.
    Take $\es$ to be the next node according to $\phi$ and go back to 3. If there is no next node, go to 6.
    \item We either have $\expos \in \tau(\es)$ or $\exneg \in \tau(\es)$, or both. If both, assign every undetermined component of $\es$ to $\bot$, as that only gives more room for completions of $\es$ to be both positive and negative. If only $\expos \in \tau(\es)$, solve the $\dap$ problem with input $(\M, \es)$. If only $\exneg \in \tau(\es)$, solve the $\dan$ problem. Then, propagate each value 0 or 1 that was  assigned to an undetermined component of $\es$ to its successors in $G$. Take $\es$ to be the next node according to $\phi$ and go back to 3.  If there is no next node, go to 6.
    \item Now that nodes have no undetermined components, check that every fact dictated by $\tau$ is true for the values that have been determined.  If all the facts are correctly satisfied, return Yes, otherwise return No.
\end{enumerate}

It is clear that, if the preceding algorithm returns Yes, then it is correct, as it has a concrete determinization consistent with $\tau$, and it must be consistent with $P$ as every undetermined component that is assigned $0$ or $1$ its propagated to the successors in the graph. It only remains to justify that it is correct when it returns No. 
Assume, looking for a contradiction, that the algorithm returns No but there actually exists a determinization $B$ of $E$  that is consistent with $\tau$ and $P$. Let $A$ be the determinization that the algorithm tested in step 6, and let $i$ be the first node according to $\phi$, the choice of the algorithm in step 1, such that $A(\phi_i) \neq B(\phi_i)$. Such an index must exists because $A$ must differ from $B$. Among all determinizations that are consistent with $\tau$ and $P$, let $B'$ be the one that maximizes the index $i$ of its first difference with $A$. Then, let $\es$ be $i$-th node according to $\phi$, and thus the first node where $A$ and $B'$ differ. If $\full \in \tau(\es)$, then the algorithm determinized $\es$ according to step 4. If $\pos \in \tau(\es)$,
the algorithm determinized $\es$ according to the algorithm for $\dap$, and thus if $B(\es)$ is effectively positive, then $A(\es)$ must also be, by Lemma~\ref{lemma:pap-iff-dap} and the theorem hypothesis.  Therefore, the inconsistency between $A$ and $\tau$ is not created by $A(\es)$. But then, as $\es$ is full, it cannot have any successors in $G$, and thus if we let $B'' \coloneqq B'$ except for $B''(\es) \coloneqq A(\es)$, then $B''$ must also be consistent with $\tau$, which contradicts the maximality of $i$. The case in which $\exneg \in \tau(\es)$ is analogous. 

It remains to see the case where $\full \not \in \tau(\es)$. Assume $j$ is the first component for which $A(\es)[j] \neq B'(\es)[j]$. 
If $A(\es)[j] = \bot$, then it must be the case that both $\expos \in \tau(\es)$ and $\exneg \in \tau(\es)$. Note that every completion of $B'(\es)$ is also a completion of $A(\es)$ and thus if $B'(\es)$ is consistent with $\tau$, then so is $A(\es)$. This implies the inconsistency in $\tau$ must appear later in $\phi$, and thus we can again take $B''$ equal to $B'$ except for $B''(\es) \coloneqq A(\es)$ which will contradict the maximality of $i$. 

If $A(\es)[j] = 1$ or $A(\es)[j] = 0$, it must be the case that $\expos \in \tau(\es)$ but $\exneg \not \in \tau(\es)$, or vice-versa. This means that  determinizations of $\es$ must have either all positive completions or all negative completions. Again because of Lemma~\ref{lemma:pap-iff-dap} and the theorem hypothesis, $A(\es)$ must hold $\allpos$ or $\allneg$ if it was possible to determinize $\es$ in that way, which is the case because $B(\es)$ does so. Then note that by taking $B''$ which is equal to $B'$ except that it $B''(\es')[j] = A(\es)[j]$ for every successor $\es'$ of $\es$ (itself included), we get an assignment that must also be consistent with $\tau$, as $\exneg \not \in \tau(\es)$ implies that $\exneg$ does not hold for any of the successors of $\es$ either. Moreover, no fact of the form $\full$ can be broken either, as for every non-full variable we already included in $F$ a guess of an undefined component for it. Thus $B''$ is consistent with $\tau$, and it either contradicts the maximality of $i$ or the minimality of $j$. Having explored all possible cases of failure, we can conclude that the algorithm is correct, and as it is clearly polynomial, we finish the proof of this claim.

\end{proof}

As the preceding claim has been proved, and there are constantly many pairs $(\tau, \gamma)$ to consider, there is a polynomial time algorithm for the whole problem.

\end{proof}

In order to finish this section, we restate and prove Proposition~\ref{prop:pap-pan-ptron}.

\begin{propbis}{prop:pap-pan-ptron}
The problems
 $\logicEvaluation(\pap,\perceptrons)$ and
$\logicEvaluation(\pan,\perceptrons)$ can be solved in polynomial time.
\end{propbis}

\begin{proof}
Based on Lemma~\ref{lemma:pap-iff-dap}, it is enough to show that the problems $\dap(\perceptrons)$ and $\dan(\perceptrons)$ can be solved in polynomial time. Let us focus on the case of $\dap$, as the other case is analogous. Thus, an input instance consists of a perceptron $\M = (w, t)$ of dimension $n$, and an undetermined instance $\es$ of dimension $n$. A polynomial time algorithm follows directly from the next claim. 
\begin{claim}
$(\M, \es)$ a \textsc{Yes} instance of $\dap(\perceptrons)$ if and only if the following equation holds
\[
 \left(\sum_{i, \es[i] \in \{0, 1\}} w_i \es[i] \right) + \left(\sum_{i, \es[i] = \bot} \min(0, w_i) \right) + 
 \left(\sum_{i, \es[i] = \unknown} \max(0, w_i) \right)
 \geq t
\]  
\end{claim}
\begin{proof}[Proof of Claim 4.]
For the forward direction, assume $(\M, \es)$ a \textsc{Yes} instance, and let $\es'$ the determinization of $\es$ such that all its completions are positive under $\M$. In particular, consider the completion $\es^*$ such that if $\es'[i] = \bot$, then $\es^*[i] = \min(0, w_i)$, and $\es^*[i] = \es'[i]$ otherwise. The fact that this completion is positive means that 
\[
 \left(\sum_{i, \es'[i] \in \{0, 1\}} w_i \es'[i] \right) + \left(\sum_{i, \es'[i] = \bot} \min(0, w_i) \right)
 \geq t
\]
Now, the components in $\es'$ can be separated according to whether they were determined or not in $\es$ already:
\begin{multline*}
\left(\sum_{i, \es[i] \in \{0, 1\}} w_i \es[i] \right) + \left(\sum_{i, \es[i] = \bot} \min(0, w_i) \right)
\ + \\
 \left(\sum_{i, \es[i] = \unknown, \es'[i] \in \{0, 1\}} w_i \es'[i] \right) + \left(\sum_{i, \es[i] = \unknown, \es'[i] = \bot} \min(0, w_i) \right)
 \geq t
\end{multline*}
By noting that $\max(0, w_i) \geq w_i \es'[i]$ and $\max(0, w_i) \geq \min(0, w_i) $ we have that
\[
\left(\sum_{i, \es[i] = \unknown} \max(0, w_i) \right) \geq  \left(\sum_{i, \es[i] = \unknown, \es'[i] \in \{0, 1\}} w_i \es'[i] \right) + \left(\sum_{i, \es[i] = \unknown, \es'[i] = \bot} \min(0, w_i) \right)
\]
and thus we conclude simply by combining the three previous equations. For the backward direction, assume the equation holds, and let define the determinization $\es^\star$ such that 
\[
\es^\star[i]\begin{cases}
    \es[i] & \text{if } \es[i] \neq \unknown\\
    1 & \text{if } \es[i] = \unknown \text{ and } w_i \geq 0\\ 0 & \text{otherwise.}
\end{cases}
\]
Note that this implies that if $e[i] = \unknown$ then $\es^\star[i] w_i = \max(0, w_i)$. Now let $\es'$ be any completion of $\es^\star$, and we aim to prove that 
\[
\sum_{i} w_i \es'[i]  \geq t
\]
By construction, we have that 
\[
\sum_{i} w_i \es'[i] = \left(\sum_{i, \es[i] \in \{0, 1\}} w_i \es[i] \right) +  \left(\sum_{i, \es[i] = \unknown} \max(0, w_i) \right) + \left(\sum_{i, \es[i]=\bot} w_i \es'[i] \right)
\]
But 
\[
\left(\sum_{i, \es[i]=\bot} w_i \es'[i] \right) \geq  \left(\sum_{i, \es[i]=\bot} \min(0, w_i) \right) 
\]
and thus
\[
\sum_{i} w_i \es'[i] \geq \left(\sum_{i, \es[i] \in \{0, 1\}} w_i \es[i] \right) + \left(\sum_{i, \es[i] = \bot} \min(0, w_i) \right) + 
 \left(\sum_{i, \es[i] = \unknown} \max(0, w_i) \right)
\]
which is at least $t$ by hypothesis. Therefore, any completion $e'$ is  positive, which concludes the proof.
\end{proof}
\end{proof}

In order to make the previous result more meaningful, we show explicitly that the class of perceptrons is not tractable for unrestricted  $\foil$.

\begin{proposition}
The problem of deciding whether a model is biased~\cite{DarwicheH20}, expressible in FOIL through the formula $\textsc{BiasedModel}$,  is $\np$-hard for the class of perceptrons.
\end{proposition}
\begin{proof}
In order to show hardness we will reduce from the subset sum problem, which is well known to be $\np$-hard. Recall that the subset sum problem consists on, given natural numbers $s_1,\ldots,s_n, k \in \mathbb{N}$, to decide whether there is a subset ~$S \subseteq \{1,\ldots,n\}$ such that
$\sum_{i \in S} s_i = k$. Let us proceed with the reduction. Based on a subset sum instance $s_1,\ldots,s_n, k$, we create a perceptron with $n$ unprotected features (that we assume to have indices $1$ through $n$) with associated weights $s_1, \ldots, s_n$ and a single protected feature, with index $n+1$ and weight $1$. Given the described weights, let $\M$ be the resulting perceptron that has those weights and bias\footnote{Recall that the bias of a perceptron has nothing to do with the notion of bias that relates to fairness.} $-k-1$. The following claim is enough to establish the reduction.

\begin{claim}
	The perceptron $\M$ is biased if and only if $s_1,\ldots,s_n, k$ is a positive instance of the subset sum problem.
\end{claim}

For the forward direction, consider $\M$ to be biased. That means there are instances $\es_1$ and $\es_2$ such that $\M(\es_1) \neq \M(\es_2)$, that differ only on the $n+1$-th feature, as it is the only protected one. Assume wlog that $\M(\es_1) = 1$ and $\M(\es_2) = 0$, by swapping the variables if it is not already the case.  This implies $\vw \cdot \es_1  \geq k+1$ and $ \vw \cdot \es_2  < k+1$. 
As $ \vw \cdot \es_1  \geq  \vw \cdot \es_2$, and $\es_1$ differs from $\es_2$ only on the $n+1$-th feature, it must hold that $\es_1[{n+1}] = 1$ and $\es_2[{n+1}] = 0$, as $w_{n+1} = 1$.
Let $P$ be the set of unprotected features of $\es_1$ (and thus $\es_2$) that are set to $1$. Then we can write $ \vw \cdot \es_1 = \big( \sum_{i \in P} s_i \big) + 1$, as each weight $w_i$ was chosen to be equal to $s_i$. Then when considering $\es_1$ we have that $\big( \sum_{i \in P} s_i \big) + 1 \geq k+1$ and, by considering $\es_2$, that  $\big( \sum_{i \in P} s_i \big) < k+1$, from which we deduce that $\sum_{i \in P} s_i = k$. We have found a subset of $\{s_1, \ldots, s_n\}$ that adds up to $k$, which is enough to conclude the forward direction of the proof.
For the backward direction, consider an arbitrary set $P \subseteq \{1, \ldots, n\}$ such that $\sum_{i \in P} s_i = k$. It is then easy to verify that the instance $\es_1$ that has a $1$ in every feature whose index belongs to $P$, a $1$ in the $n+1$-th feature, and $0$ on the rest, is a positive instance of $\M$. Furthermore, $\es_2$ that differs from $\es_1$ only in the $n+1$-th feature can be checked to be a negative instance. As we have found a pair of instances that differ only on protected features, and yet have opposite classifications, the model $\M$ must be biased.
\end{proof}

We now restate and prove Proposition~\ref{prop:hardness-dt-ef}.

\begin{propbis}{prop:hardness-dt-ef}
Let $\C$ be $\obdd$ or $\dt$. The problems
 $\logicEvaluation(\pap,\C)$ and
$\logicEvaluation(\pan,\C)$ are $\np$-hard.
\end{propbis}

\begin{proof}
. It is enough to prove that hardness holds already for the class of ordered decision trees (i.e., $\dt \cap \obdd$). Moreover, based on Lemma~\ref{lemma:pap-iff-dap}, it is enough to show that the problems $\dap$ and $\dan$ are $\np$-hard for ordered decision trees. We focus on the case of $\dap$, as the other one is analogous.

We do this by reducing from 3-SAT. Indeed, let $\varphi$ be a formula in 3CNF, with $m$ clauses and $n$ variables. Let us assume that $m = 2^k$ for some integer $k$, as otherwise one can simply add $2^{\lceil \log_2 m \rceil} - m \in O(m)$ clauses consisting of new fresh variables. Then, create an ordered decision tree $\T$ of dimension $k+n$ in the following way:
\begin{itemize}
    \item The first $k$ features are labeled $a_1, \ldots, a_k$, and then $n$ features corresponding to the variables of $\varphi$ are labeled $x_1, \ldots, x_n$. The features in $\T$ are ordered 
    \[a_1 \preceq a_2 \preceq \cdots \preceq a_k \preceq x_1 \preceq x_2 \preceq \cdots \preceq x_n. \]
    \item Start creating $\T$ by building a complete binary ordered tree over the features $a_1, \ldots, a_k$, where the root has label $a_1$, and all nodes at distance $i$ from the root have label $a_{i+1}$. The last layer of said tree consists exactly of $2^{k-1}$ nodes labeled  $a_{k}$. 
    \item For each clause $C_i \quad (1 \leq i \leq 2^k)$ create an ordered (according to the ordering described above) decision tree $\T_i$ equivalent to said clause. Note that as each clause mentions exactly 3 variables, each of the $\T_i$ can be built in constant time from clause $C_i$.
    \item Let $L$ be the set of $2^{k-1}$ nodes labeled with $a_k$ in $\T$. Let $\ell_1, \cdots, \ell_{2^{k-1}}$ be any ordering of $L$, and to node $\ell_i$ connect $\T_{2i-1}$ with an edge labeled $0$ and $\T_{2i}$ with an edge labeled $0$. 
\end{itemize}
Note that this construction can trivially be performed in polynomial time. Now build an undetermined instance $\es$ of dimension $k+n$, where $\es[i] = \bot$ for $1 \leq i \leq k$ and $\es[i] = \unknown$ otherwise. We claim that there exists a determinization $\es'$ of $\es$ such that all its completions are positive, if and only if, $\varphi$ is satisfiable. Indeed, assume first that such a determinization $\es'$ exists. Then, define $\es''$ in the following way:
\[
\es''[i] = \begin{cases}
\es[i] & \text{if } \es[i] \neq \unknown\\
\es'[i] & \text{if } \es[i] = \unknown \text{ and } \es'[i] \in \{0, 1\}\\
0 & \text{otherwise.}
\end{cases}
\]
As by construction $\es' \subseteq \es''$ it must also hold that all completions of $\es''$ are positive. Moreover, note that $\es''[i] = \bot$ for all $1 \leq i \leq k$ and $\es''[i] \in \{0, 1\}$ for all $k+1 \leq i \leq k+n$. We build an $\sigma$ of variables of $\varphi$ based on $\es''$ by setting variable $x_i$ to $\es''[k+i]$, for $1 \leq i \leq n$. Now, we claim that $\sigma$ is a satisfying assignment. Indeed, to see that $\sigma$ satisfies clause $C_i$, consider the completion $\es^*$ of $\es''$ that sets the features $a_1, \ldots, a_k$ in such a way that the path of $\es^*$ over $\T$ arrives to $\T_i$. As $\es^*$ is a positive instance of $\T$, it must be a positive instance of $\T_i$, and thus by construction $\sigma$ satisfies $C_i$.

For the other direction, let $\sigma$ be a satisfying assignment to $\varphi$, and build the determinization $\es^\star$ such that $\es^\star[k+i] = \sigma(x_i)$ for $1 \leq i \leq n$. As $\sigma$ satisfies every clause, $\es^\star$ is a positive instance of every $\T_i$, and thus any completion of $\es^\star$ is a positive instance of $\T$.
\end{proof}

\section{Proof of Theorem \ref{theo:kcobdd}}
\label{sec:tractable-kcobdd}
In order to make the proof more readable, let us first prove a lemma about simple operations over $\kcobdd$s.

 \begin{lemma}
The following operations can be performed in polynomial time:
\begin{itemize}
    \item \textbf{(Negation)} Given a $\kcobdd$ $\M$ of dimension $n$, compute a $\kcobdd$ $\neg \M$ over the same ordering of variables, such that $\neg \M(\es) = 1 - \M(\es)$ for every instance $\es$ of dimension $n$.
    \item \textbf{(Disjunction)} Given $\kcobdd$s $\M_1$ and $\M_2$, of dimension $n$, with a common linear ordering ${<}$ on the set $\{1, \ldots, n\}$ compute a COBDD $\M \coloneqq \M_1 \lor \M_2$ over the same linear ordering ${<}$, and width at most $2k$,  such that $\M(\es) = \max(\M_1(\es), \M_2(\es))$ for every instance $\es$ of dimension $n$.
    \item \textbf{(Conjunction)} Given $\kcobdd$s $\M_1$ and $\M_2$, of dimension $n$, with a common linear ordering ${<}$ on the set $\{1, \ldots, n\}$ compute a COBDD $\M \coloneqq \M_1 \land \M_2$ over the same linear ordering ${<}$, and width at most $2k$,  such that $\M(\es) = \min(\M_1(\es), \M_2(\es))$ for every instance $\es$ of dimension $n$.
\end{itemize}
\end{lemma}

\begin{proof}
Negation is trivial, it suffices to interchange the labels $\true$ and $\false$ in every leaf of $\M$. Disjunction and Conjunction follow the classical algorithm for OBDDs by Bryant \cite{10.1109/TC.1986.1676819}, and in what follows we argue that width is no more than doubled.  Let us introduce some notation; the inductive structure of OBDDs allows us to say that $\M_1$ has a root node labeled with $r \in \{1, \ldots, n\}$ connected to OBDDs $\M_1^0$ and $\M_1^1$ by edges labeled with $0$ and $1$ respectively, which we denote as $\M_1 = (r, \M_1^0, \M_2^0)$. Analogously, let $\M_2 = (r, \M_2^0, \M_2^1)$, as $\M_1$ and $\M_2$ share the ordering ${<}$ and are complete, their root must have the same label. Then, for $\mathrm{op} \in \{\text{Conjunction}, \text{Disjunction}\}$, Bryant's algorithm inductively computes operations according to the following equation
\begin{equation}
\label{eq:bryant}
    \mathrm{op}(\M_1, \M_2) = \left(r, \mathrm{op}(\M_1^0, \M_2^0), \mathrm{op}(\M_1^1, \M_2^1)\right).
\end{equation}

Let us use notation $\M\to j$ for the number of nodes in $\M$ labeled with $j$. 
We are now ready to prove a stronger claim by induction on $n$, the dimension of the models, from which the lemma trivially follows. 

\begin{claim}
\label{claim:cobbds}
Let $\M_1$ and $\M_2$ be COBDDs of dimension $n$ with a common ordering ${<}$, and let $j$ be any label in $\{1, \ldots, n\}$. Then Bryant's algorithm guarantees that \[
\mathrm{op}(\M_1, \M_2) \to j \leq (\M_1 \to j) + (\M_2 \to j).\]
\end{claim} 
\begin{proof}[Proof of Claim~\ref{claim:cobbds}]
If $n = 1$, the claim is trivial, so we assume $n > 1$. Trivially, $\mathrm{op}(\M_1, \M_2)\to r = 1$, so the claim is also trivial for $j = r$. We now examine the general case of $j \neq r$, for which we will use the inductive hypothesis of $n-1$. Based on Equation~\ref{eq:bryant}, we have that
\[
    \mathrm{op}(\M_1, \M_2)\to j = (\mathrm{op}(\M_1^0, \M_2^0)\to j) + (\mathrm{op}(\M_1^1, \M_2^1) \to j)
\]
Note that $\mathrm{op}(\M_1^0, \M_2^0)$ and $\mathrm{op}(\M_1^1, \M_2^1)$ are COBDDs of dimension $n-1$, as they do not include label $r$. Then, by inductive hypothesis, we have that 
\[
(\mathrm{op}(\M_1^0, \M_2^0)\to j) + (\mathrm{op}(\M_1^1, \M_2^1) \to j) \leq (\M_1^0 \to j) + (\M_2^0 \to j) + (\M_1^1 \to j) + (\M_2^1 \to j)
\]
But by definition,
\[
\M_i \to j = (\M_i^0 \to j) + (\M_i^1 \to j) \qquad  (i \in \{1, 2\}).
\]
By combining the three previous equations, we get the desired result:
\[
\mathrm{op}(\M_1, \M_2)\to j \leq (\M_1 \to j) + (\M_2 \to j).
\]
\end{proof}

We are now ready to finish the proof of the lemma. As for a model $\M$ of dimension $n$  we have that $\mathrm{width}(\M)= \max_{1 \leq j \leq n}(\M \to j)$, it follows from Claim~\ref{claim:cobbds} that
\[
\mathrm{width}\left(  \mathrm{op}(\M_1, \M_2)\right) \leq \max_{1 \leq j \leq n} (\M_1 \to j) + (\M_2 \to j) \leq \mathrm{width}(\M_1) + \mathrm{width}(\M_2)
\]
which concludes the proof.
\end{proof}

We now state a lemma of Capelli and Mengel~\cite{DBLP:journals/corr/abs-1807-04263, capelli_et_al:LIPIcs:2019:10257} that will be used in our proof, but before let us introduce appropriate notation. Given a COBDD $\M$ of dimension $n$,  and a set $S \subseteq \{1, \ldots, n\}$ we define $\exists_S \M$ as a COBDD of dimension $n- |S|$, such that for every instance $\es$ of dimension $n- |S|$, we have that
$ \exists_S \M (\es) = 1$ iff there is an instance $\es'$ of dimension $n$ that holds both:  
\begin{itemize}
    \item $\M(\es') = 1$
    \item Let the list $t_1, \ldots, t_{n-|S|}$ correspond to $\{1, \ldots, n\} \setminus S$ in increasing order. Then $e'[t_i] = e[i]$, for every $i \in \{1, \ldots, n - |S|\}$.
\end{itemize}

For example, if $\M$ is a model of dimension $3$ equivalent to $(x_1 \land x_2) \lor (x_2 \land \neg x_3)$, then if we take $S = \{1, 3\}$, $\exists_S \M$ is equivalent to $x_2$, as if $x_2$ is true, then there exists values of $x_1$ and $x_3$ that satisfy $\M$ (namely $1$ and $0$) respectively, whereas if $x_2$ is false, then no values of $x_1$ and $x_3$ will help satisfy $\M$.

\begin{lemma}[Lemma 1, \cite{DBLP:journals/corr/abs-1807-04263}]
\label{lemma:capelliForget}
Fix an integer $k > 0$. Given a COBDD $\M$ of dimension $n$ and width $k$, and a set $S \subseteq \{1, \ldots, n\}$, one can compute a COBDD $\exists_S\M$ of width at most $2^k$ in polynomial time.
\end{lemma}

We can define $\forall_S \M$ as $\neg \exists_S \neg \M$, and thus the previous lemma applies as well to $\forall_S \M$.

In our case, however, partial instances can have three possible values: $0, 1,$ and $\bot$. Therefore, we define \emph{Complete Ordered Ternary Decision Diagrams} (COTDDs) analogously to COBDDs but with nodes having three outgoing edges labeled with $0, 1, \bot$. 
Note that, given a COBDD $\M$ of dimension $n$, we can  build in polynomial time a COTDD $\M^3$ of dimension $n$ such that for every partial instance $\es$ of dimension $n$, $\M^3(\es) = 1$ iff $\es$ is a full instance and $\M(\es) = 1$. This can be done by first creating a path $P$ of nodes according to the underlying order of $\M$, starting from the second label in the ordering, such that each node is connected to the next one by its three outgoing edges, and the last one is connected to a leaf labeled $\false$. Then, start by $M^3 \coloneqq M$, and to each node labeled $u$ in $M^3$, connect it with its outgoing $\bot$ edge to the node labeled with the successor of $u$ in $P$.

We now state the equivalent lemmas for COTDDs.

 \begin{lemma}
 \label{lemma:cotdd}
The following operations can be performed in polynomial time:
\begin{itemize}
    \item \textbf{(Negation)} Given a COTDD $\M$ of dimension $n$ and width $k$, compute a COTDD $\neg \M$ of width $k$ and the same ordering of variables, such that $\neg \M(\es) = 1 - \M(\es)$ for every instance $\es$ of dimension $n$.
    \item \textbf{(Disjunction)} Given COTDDs  $\M_1$ and $\M_2$ of width at most $k$, of dimension $n$, with a common linear ordering ${<}$ on the set $\{1, \ldots, n\}$ compute a COTDD $\M \coloneqq \M_1 \lor \M_2$ over the same linear ordering ${<}$, and width at most $3k$,  such that $\M(\es) = \max(\M_1(\es), \M_2(\es))$ for every instance $\es$ of dimension $n$.
    \item \textbf{(Conjunction)} Given COTDDs  $\M_1$ and $\M_2$ of width at most $k$, of dimension $n$, with a common linear ordering ${<}$ on the set $\{1, \ldots, n\}$ compute a  COTDD $\M \coloneqq \M_1 \land \M_2$ over the same linear ordering ${<}$, and width at most $3k$,  such that $\M(\es) = \min(\M_1(\es), \M_2(\es))$ for every instance $\es$ of dimension $n$.
\end{itemize}
\end{lemma}

\begin{proof}
The case of negation is exactly as before. For disjunction and conjunction the proof works in the same way but by considering that if  $\M_1= (r, \M^0_1, \M^1_1, \M^\bot_1)$ and $\M_2= (r, \M^0_2, \M^1_2, \M^\bot_2)$, then the following  equations hold:
\begin{equation}
     \mathrm{op}(\M_1, \M_2) = \left(r, \mathrm{op}(\M_1^0, \M_2^0), \mathrm{op}(\M_1^1, \M_2^1), \mathrm{op}(\M_1^{\bot}, \M_2^{\bot})\right)
\end{equation}
 \begin{equation}
  \M_i\to j = (\M_i^0\to j) + (\M_i^1 \to j) + (\M_i^\bot \to j) \qquad (i \in \{1, 2\}, \forall j \neq r, 1 \leq j \leq \dim(\M))  
 \end{equation}
\end{proof}

\begin{lemma}[~Lemma 1, \cite{DBLP:journals/corr/abs-1807-04263}]
\label{lemma:capelliForget3}
Fix an integer $k > 0$. Given a COTDD $\M$ of dimension $n$ and width $k$, and a set $S \subseteq \{1, \ldots, n\}$, one can compute a COBDD $\exists_S\M$ of width at most $2^k$ in polynomial time.
\end{lemma}
\begin{proof}
Direct from the proof in~\cite{DBLP:journals/corr/abs-1807-04263}, as the same construction can be done in the ternary case. 
\end{proof}
We are finally ready to restate our theorem and prove it.

\begin{thmbis}{theo:kcobdd}
Let $k \geq 1$ and query $\varphi$ in $\foil$. Then 
$\logicEvaluation(\varphi,\kcobdd)$ 
can be solved in polynomial~time.
\end{thmbis}

\begin{proof}
We assume that $\varphi$ alternates quantifiers and starts with an existential one without loss of generality as one can trivially enforce this by adding \emph{dummy} variables. Thus, let $\varphi(x_1, \ldots, x_\ell) = \exists x_{\ell + 1} \forall x_{\ell+2} \cdots \forall x_{\ell +m -1} \exists  x_{\ell + m} \psi(x_1, \ldots, x_{\ell + m})$, where $\psi$ is quantifier-free, and let $\M, \es_1, \ldots, \es_\ell$ be an input of the $\logicEvaluation(\varphi,\kcobdd)$ problem.  Let $n = \dim(\M)$.

Let us introduce a final piece of notation: from a list of partial instances $\es_1, \ldots, \es_{p}$ of dimension $n$ each, we can define a unique instance $\es_{[1, p]}$ of dimension $pn$ that is simply the concatenation of instances $\es_1, \ldots, \es_p$. More in general, we will use notation $[i, j]$ for $i \leq j$ to denote the set $\{i,  \ldots, j\}$. We use as well $\es[i:j]$ with $i < j$ referring to the partial instance $(\es[i], \ldots, \es[j])$ of dimension $j-i+1$, where $\es$ is a a partial instance of dimension at least $j$.

Now, the proof consists of two parts. First, we will show that based on $\M$ one can build a COTDD $\M'$ of width at most $f(k)$ for a suitable function $f$, such that 
\[
 \M \models \psi(\es_1, \ldots, \es_{\ell+m})  \iff \M'(\es_{[1, \ell+m]}) = 1
\]
We will do this by induction over $|\psi|$ in the next claim, but first let us define  a  linear ordering ${\prec}$ of $[1, n(\ell + m)]$ as follows. If $\pi(1), \ldots, \pi(n)$ is the ordering of $\M$, then
\[
\pi(1) \prec \pi(1)+n \prec \pi(1)+2n \prec \ldots \prec \pi(1) + (\ell+m-1)n \prec
\]
continued by
\[
\pi(2) \prec \pi(2)+n \prec \pi(2)+2n \prec \ldots \prec \pi(2) + (\ell+m-1)n \prec
\]
and so on, all the way up to
\[
\pi(n) \prec \pi(n) + n \prec \pi(n) + 2n \prec \ldots \prec  \pi(n) + (\ell+m-1)n
\]
We now formalize the desired claim.
\begin{claim}
Let $\phi$ be any formula in $\foil$ mentioning a set of variables  $\{x_i, \ldots, x_j\}~\subseteq~\{x_1, \ldots, x_{\ell+m}\}$,  that has at most $c$ logical connectives (i.e., $\land, \lor, \neg$). 
Then, we can build in polynomial time a COTDD $\M_\phi$ over the ordering ${\prec}$, of dimension $n(\ell +m)$ and width at most $f(c, k)$ for a suitable function $f$,  such that for any partial instance $\es$ of dimension $n(\ell + m)$
\[
    \M_\phi(\es) = 1 \iff \M \models \phi(\es[i: i+n-1], \ldots, \es[j:j+n-1])
\]
\end{claim}
\begin{proof}
The proof is by induction on $c$. The base cases $(c=0)$ are constructive and relatively involved, so let us start by the inductive cases, where $c > 0$.

\begin{itemize}
     
     \item If $\phi = \phi_1 \lor \phi_2$, simply use the inductive hypothesis to build models $\M_{\phi_1}$ and $\M_{\phi_2}$ and then use Lemma~\ref{lemma:cotdd} to build $\M_\phi \coloneqq \M_{\phi_1} \lor \M_{\phi_2}$ of width is at most $3\max(\width(\M_{\phi_1}, \width(\M_{\phi_2})))$. It is not hard to see that the resulting model satisfies the desired conditions.
     
      \item  If $\phi = \phi_1 \land \phi_2$, simply use the inductive hypothesis to build models $\M_{\phi_1}$ and $\M_{\phi_2}$ and then use Lemma~\ref{lemma:cotdd} to build $\M_\phi \coloneqq \M_{\phi_1} \land \M_{\phi_2}$ of width is at most $3\max(\width(\M_{\phi_1}, \width(\M_{\phi_2})))$. It is not hard to see that the resulting model satisfies the desired conditions.
      
       \item If $\phi = \neg \phi_1$, simply use the inductive hypothesis to build model $\M_{\phi_1}$ and  then use Lemma~\ref{lemma:cotdd} to build $\M_\phi \coloneqq \neg \M_{\phi_1} $ of width is at most $\width(\M_{\phi_1})$. It is not hard to see that the resulting model satisfies the desired conditions.

\end{itemize}

For the base cases let us introduce some notation to facilitate our construction. We will use $P[u, v]$, for $u \prec v$, to mean a path of nodes labeled from $u$ up to $v$ following the ordering ${\prec}$ and where each node is connected to the next one by its three outgoing edges. Also, if $a$ and $b$ are nodes, $a \tox{x} b$ with $x \in \{0,1,\bot\}$ means that $a$'s outgoing edge labeled with $x$ goes to $b$. If $P$ is a path of nodes, then $P \tox{x} \alpha$ means that the outgoing edge labeled with $x$ of the last node in $P$ goes to $\alpha$. Similarly, $\alpha \tox{x} P$ means means the outgoing edge labeled with $x$ of $\alpha$ (which could be the last node of a path if $\alpha$ is one) goes to the first node in $P$. We can now prove that the base cases, when $c = 0$, are also satisfied.
\begin{itemize}
        \item  If $\phi = \pos(x_i)$ for $i \in [1, \ell +m]$, then we build $\M_\phi$ recursively as follows. Assume $\M = (r, \M^0, \M^1)$, and create a path $P_r = P[r, r+(i-1)n]$. Then create three identical paths $P_{r\to0} =  P_{r\to1} = P_{r \to f} \coloneqq P[r+in, \pi(n) + (\ell+m-1)n]$. Then, add connections $P_r \tox{\bot} P_{r \to f} \tox{0,1,\bot} \false$, as positive instances have no occurrences of $\bot$. Connect $P_r \tox{0} P_{r\to 0}$ and $P_r \tox{1} P_{r \to 1}$.  Finally, if we let $R(\M^0), R(\M^1)$ be the results of the recursive procedure applied to $\M^0$ and $\M^1$, respectively, then we connect $P_{r\to0} \tox{0,1,\bot} R(\M^0)$ and $P_{r\to1} \tox{0,1,\bot} R(\M^1)$. The recursive procedure applied to a leaf will just keep it as such. It is easy to see that this can be done in polynomial time, and it is not hard to see that  the desired equation for $\M_\phi$ is satisfied. Moreover, note that for each node $r$ in $\M$, we introduce three nodes with the same label (namely, when creating the paths $P_{r\to0}$, $P_{r\to1}$ and $P_{r\to f}$, noting that  nodes in $P_{r\to f}$ can be shared for every $q$ such that $r \prec q$) which implies the width of $\M_\phi$ is no more than $3k$.
     \item If $\phi = x_i \subseteq x_j$, we build $\M_\phi$ as follows. We first build a path $P_f = [1, \pi(n) + (\ell+m-1)n]\to \false$.
     For each $1 \leq t \leq n$ we will build a gadget that checks that $\es_i[t] = \bot \lor \es_i[t] = \es_i[t]$. The gadget for $\pi(1)$ will be connected to that of $\pi(2)$ and so on, so the ordering ${\prec}$ is respected.
     The exact form of each gadget depends on whether $i < j$ or the opposite, as in the first case $t + (i-1)n \prec t + (j-1)n$, for each $1 \leq t \leq n$, and vice-versa if $j < i$. Consider first the case where $i < j$. Then, for the gadget for $t \in \{1, \ldots, n\}$ we build a path $P^1 \coloneqq P[\pi(t), \pi(t)+(i-1)n]$,  two identical paths $P^2 = P^3 \coloneqq P[\pi(t) + in, \pi(t) + (j-1)n]$. and $P^4 \coloneqq P[\pi(t) + jn, \pi(t) + (\ell+m-1)n]$. We then connect $P^1\tox{0}P^2\tox{0}P^4$ and $P^1\tox{1}P^3\tox{1}P^4$, Also, for a label $u$, let us denote $s(u)$ to its successor according to $\prec$, and let $f(s(u))$ be the node in $P_f$ with label $s(u)$. Next, connect $P^1\tox{1}P^3\tox{0}f(s(\pi(t) + (j-1)n))$ and $P^1\tox{0}P^3\tox{1}f(s(\pi(t) + (j-1)n))$. If $t=n$, we connect $P^4\tox{0,1,\bot}\true$ and otherwise, if $t< n$ and $G_{t+1}$ is the gadget for $t+1$, we connect $P^4\tox{0,1,\bot}G_{t+1}$.  The case when $j < i$ is similar, but every gadget for $1 \leq t \leq n$ checks first the value of the $t$-th feature in the $j$-th variable, and then checks that the $t$-th variable is either $\bot$ or matches the previously mentioned value. We omit the details as they are a trivial modification of the previous construction.
     
     It is clear that this procedure takes polynomial time, and it is not hard to see that this procedure satisfies the desired condition. Moreover the resulting COTDD has width $3$, as $P^2$ and $P^3$ make for two occurrences of the labels appearing in them, and a final one comes from $P_f$.
\end{itemize}

As each recursive step reduces $c$ by one, and can at most increase the width by a factor of $3$, it follows that the width of the resulting model is at most $k3^c$, and thus $f(c,k) = k3^c$ is a suitable function. This concludes the proof of the claim.
\end{proof}

Now, for the second part of the proof, assume we have already built $\M'$ by using the previous claim noting that as $\psi$ is a fixed formula, $c$ is a fixed constant which implies that the function $f$ of the previous claim depends thus solely on $k$.  Let us now state a simpler claim.
\begin{claim}
For appropriate sets $S_1, \ldots, S_m$ subsets of $[1, n(\ell+m)]$, that can be determined in polynomial time, the following holds:
\[
\M \models \varphi(\es_1, \ldots, \es_\ell) \iff \exists_{S_1} \forall_{S_2} \cdots \forall_{S_{m-1}} \exists_{S_m} \M'(\es_{[1, \ell]}) = 1
\]
\end{claim}
\begin{proof}
Trivial by the definition of $\exists_S$ and $\forall_S$ when defining each $S_i$ as $[n(\ell+i-1)+1, n(\ell+i)]$.
\end{proof}

In order to finish the proof, we use the previous claim and simply compute $\M^\star \coloneqq \exists_{S_1} \forall_{S_2} \cdots \forall_{S_{m-1}} \exists_{S_m} \M'$ by $m$ repeated applications of Lemma~\ref{lemma:capelliForget3}, and computing two negations for each $\forall_{S_i}$ according to Lemma~\ref{lemma:cotdd}. This results in $\M^\star$ having width at most $f(k) = 2^{2^{2^{\cdots^{2^{3^{|\varphi|k}}}}}}$, where the tower has $m$ times the number $2$. Then we build the partial instance $\es_{[1, \ell]}$ and finally evaluate $\M^\star(\es_{[1, \ell]})$, which thanks to the previous claim is enough to solve the whole problem. As every part of the algorithm is proven to be correct, and the running time of each component is polynomial, we conclude the whole proof.

\end{proof}

\section{Details of the FOIL implementation and the experimental setting}
\label{sec:supp-implementation}

All our code and instructions to run experiments can be found at the following URL 
\begin{center}
    \url{https://github.com/AngrySeal/FOIL-Prototype}
\end{center}

For the implementation of the algorithms we used C++ to assure efficiency. 
For parsing queries we used the ANTLR (v4.9.1) parser generator. 
Queries can be specified in a straightforward way in plain text by using tokens \texttt{Exists} and \texttt{ForAll} for quantification, tokens \verb+~+, \texttt{\^}, \texttt{V}, for logical connectives $\neg$, $\wedge$ and $\vee$. respectively, and tokens of the form \texttt{x1}, \texttt{x2}, \texttt{x3}, \texttt{u}, \texttt{v}, \texttt{z}, etc. for variables, plus parentheses for stating precedence.
Instances mentioned in the queries are written as \texttt{(0,1,?,0)} where \texttt{?} represents the $\bot$ value (for defining partial instances).
Besides that, we used \texttt{P( )} for the unary operator $\pos(\cdot)$ and $\texttt{<=}$ for the containment $\subseteq$.
The following is an example query.
\begin{verbatim}
Exists x, Exists y, (P(x) V P(y)) ^ (~( x <= y ) ^ ~(y <= (?,?,?,0,1,?,?)))
\end{verbatim}

For debugging purposes we implemented a naive evaluation method (126 lines of code) that considers models as black boxes for evaluating $\pos$, and that evaluates a query by testing all possible combinations for the mentioned variables.
For the obvious reasons, this implementation is not practical but it is straightforward to prove its correctness. 
Thus we use it to check the correctness of the evaluation process of the more efficient algorithms.

We implemented versions for the query evaluator for perceptrons (not described in the paper), decision trees and a modification of FBDDs (see Section~\ref{sec:binarization} for details).
For these last two cases, the implementation had 660 lines of code.
Trees and BDDs are passed to the implementation in a straightforward JSON format.
We checked the correctness by generating a set of random queries over random models and comparing the output of each algorithm against our naive implementation.

\section{Details of the experimental setting}
\label{sec:experimentation}

We used Python for the query and models generation process.
We generated random queries with the following recursive process.
We initially fix the dimension of the queried model, the number of quantified variables allowed to be used in the complete query, and whether they are going to be universally or existentially quantified.
We then construct the quantified-free part as follows.
When asking for an expression of size $n$, the query generator method chooses a random size $k$ from $1$ to $n-1$, then generates two expressions of size $k$ and $n-k$ and joined them choosing either \texttt{\^} or \texttt{V} randomly.
The base case is when asking for an expression of size $1$ in which case we choose randomly between \texttt{P($x$)}, \texttt{$C$ <= $x$}, \texttt{$x$ <= $C$} and \texttt{$x$ <= $y$}, where $C$ represents a random partial instance constructed according to the dimension of the queried model as a tuple using values \texttt{0}, \texttt{1} and \texttt{?}, and $x$ and $y$ represents randomly chosen variables from all of the variables allowed to be used in the query. 
Every random choice was done with numpy's \texttt{default\_rng}.
Before returning from every recursive call, the method choose randomly whether a negation (\verb+~+) is added in fron of the expression. 

For generating the decision-tree models to be queried, we used the Scikit-learn library.
We first select an input dimension, and then we generated a random dataset of that dimension.
All input data that we generate are random binary tuples and the target value (classification) is a random bit.
Then we select the size $N$ for the tree and trained a decision-tree model with $N$ leaves.
Finally, we transformed the obtained decision-tree into a binary one in the JSON format that our implementation can consume.

For the experiments shown in Figure~\ref{fig:avg_time} and~\ref{fig:max_time} and using the methods described above, we generanted $60$ random queries and trained $24$ random decision trees of different sizes (see Section~\ref{sec:core-implementation}).
The performance tests were done in a small personal computer: 64-bits, 2.48GHz, Dual Core Intel Celeron N3060 with 2GB of RAM and Linux Mint 20.1 Ulyssa.
Even in this modest machine, the evaluation time for random queries and models was extremely short (see Section~\ref{sec:core-implementation}).
This gives evidence that our methods can be run even for trees and queries of considerable size in a personal machine without the need of a big computer infrastructure.

\section{A high-level language for FOIL}
\label{sec:high-level}

As we described in the body of the paper, we designed a high-level user-friendly syntax tailored for general models with numerical and categorical features.
As \foil\ does not allow the use of features beyond binary ones, we need to develop a way for binarizing queries and models. 
We describe the binarization in the next section, and we only describe here the main features of our user-friendly language.

Figure~\ref{fig:hlex} shows examples of queries written for the Student Performance Data Set~\cite{student}\footnote{Download dataset at: \url{https://archive.ics.uci.edu/ml/datasets/Student+Performance}}.
In the high-level syntax we use the expressions \texttt{exists} and \texttt{for every} that represents the logical quantification, and tokens \texttt{and}, \texttt{or}, \texttt{not} and \texttt{implies} for typical logical connectives.
Variables can be any string that is not a reserved keyword. 
In the examples in Figure~\ref{fig:hlex} we use \texttt{student}, \texttt{st1} and \texttt{st2} as variables.

Our implementation allows for loading a trained model before evaluating queries (see details on model loading below).
Whenever a model is loaded the meta information about features and types as well as the classification is also loaded to be interpreted in the queries.
A main difference with basic \foil\ is that in our high-level syntax we allow for the use of named features.
For example in the first query in Figure~\ref{fig:hlex} we use \texttt{student.male} to refer to the binary feature \texttt{male} of a student instance.
Moreover we can refer to different classes (the output of the model) with names.
For our example, we trained a binary classifier in which the positive class name is \texttt{goodFinalGrades}.
Thus, the expression \texttt{goodFinalGrades(student)} is equivalent to the \texttt{P( )} expression in the basic \foil\ implementation.
With all this we can intuitively interpret the first query in Figure~\ref{fig:hlex} as asking if having a male gender is enough for the model to make a decision about the final grades.

Besides naming features, our syntax also allows the comparison of with numerical thresholds. 
In this case we use the typical  \texttt{<=} and \texttt{>}  with their natural meaning.
For instance, in our example using the Student Performance Data Set, the feature \texttt{alcoholWeek} states the level of alcohol consumption during week days, with value $1$ begin low, and value $5$ being high.
Thus, the second query in Figure~\ref{fig:hlex} asks if it is possible that the trained model classify a student with a high alcohol consumption during week days (\texttt{student.alcoholWeek > 3}) as having good final grades.
The comparison with numerical thresholds departs significantly from the base \foil\ formalization, and thus it is not trivial how to compile this type of queries into \foil. 
We describe the process in the next section.

Finally, in order to have meaningful answers for existential queries of the form $\exists x (\neg \pos(x))$, that is, asking for the existence of instances that are classified as negative for the model, we implemented the operator \texttt{full( )} that essentially requires an instance to be not partial, that is, not using $\bot$ (the formal definition of this property is in Equation~\ref{eq-def-full}).
The reason for this is that only total instances can be positive in models and since queries are evaluated over the set of all partial instances, the query $\exists x (\neg \pos(x))$ is trivially true.
Having \texttt{full( )} as part of the language allows us to more easily deal with this case.
For example, the third query in Figure~\ref{fig:hlex} uses \texttt{full( )} to ask if there is a student with low alcohol consumption during the week and that is classified as not getting good grades.

All our queries have as possible answer either \texttt{YES} or \texttt{NO}.
It is not difficult to extend our implementation such that, whenever the answer for an existential query is a \texttt{YES} then we can provide an instance as witness for that answer.
One can similarly provide a witness when the answer for a universal query is a \texttt{NO}.
This is part of our ongoing work.

Finally, we handcrafted over 20 queries similar to the ones in Figure~\ref{fig:hlex} and tested them over a decision tree with no more than 400 leaves.
You can find the complete set of queries that we tested in our companion code.

\begin{figure}[t!]
\begin{verbatim}
for every student, 
    student.male = true
    implies goodFinalGrade(student)
    
\end{verbatim}

\begin{verbatim}
exists student, 
    student.alcoholWeek > 3
    and goodFinalGrade(student)
    
\end{verbatim}

\begin{verbatim}
exists student, 
    student.alcoholWeek < 2
    and full(student) and not goodFinalGrade(student)
    
\end{verbatim}

\begin{verbatim}
for every student, 
    student.alcoholWeekend > 3 and student.alcoholWeek > 3
    implies not goodFinalGrade(student)
    
\end{verbatim}

\begin{verbatim}
exists student, 
    (student.alcoholWeekend > 3 or student.alcoholWeek > 3)
    and student.gradePartial2 <= 6 and student.male = false
    and goodFinalGrade(student)
    
\end{verbatim}

\begin{verbatim}
exists st1, exists st2, 
    st1.studyTime > 2 and st2.studyTime <= 3
    and goodFinalGrade(st1) and 
    full(st2) and not goodFinalGrade(st2)
    
\end{verbatim}
\caption{Example queries in high-level syntax}
 \label{fig:hlex}
\end{figure}

\section{Binarization of queries and models}
\label{sec:binarization}

The definition of $\foil$ considers only binary instances (i.e., tuples in $\{0,1\}^n$ for some $n \geq 1$), and consequently, binary classifiers $\M: \{0, 1\}^n \to \{0, 1\}$ for some $n \geq 1$. As many real life classification problems involve a combination of categorical and numerical features, this could present a limitation to our approach. However, in this section we show that it is possible to overcome this apparent drawback by binarizing queries and models. A recent article by Choi et al.~\cite{DBLP:journals/corr/abs-2007-01493}, studies binary encodings for decision trees as well. Our approach is slightly different, as we are concerned as well with the issue of binarizing queries.

Let us define $\hfoil$ as a \emph{high-level} equivalent of $\foil$.
First, we define a \emph{schema} for $\hfoil$ as a mapping from \emph{feature names} to \emph{feature types}, which can be either $\mathbb{R}$ or $\mathbb{B} = \{0, 1\}$, intuitively meaning that said feature is numerical or Boolean, respectively. We use notation $S(f)$ to obtain the type of a feature by its name. Moreover, if a schema $S$ defines the type of a feature $f$, we say $f \in S$. 

For example, consider the following possible schema for the  Student Performance Data Set:

\[
 S = \{(\text{age}, \mathbb{R}), (\text{alcoholWeek}, \mathbb{R}), (\text{parentsTogether}, \mathbb{B}),  \ldots\}.
\]

We say a real-valued decision tree $\T$ is compatible with a schema $S$ if each internal node $u \in \T$ holds one of the following conditions:
\begin{itemize}
    \item \textbf{(Numerical)} Node $u$ has label $(f, \tau)$ for some $f \in S, \tau \in \mathbb{R}$, and $S(f) = \mathbb{R}$.
    \item \textbf{(Boolean)} Node $u$ has label $f$ with $f \in S$, and $S(f) = \mathbb{B}$.
\end{itemize}

Given a schema $S$, the following are atomic $\hfoil_S$ formulas:
\begin{itemize}
    \item $\pos(x)$, where $x$ is a variable.
    \item $\full(x)$, where $x$ is a variable.
    \item $(\leq, x, f, \tau)$, where $x$ is a variable, $S(f) = \mathbb{R}$,  and $\tau \in \mathbb{R}$.
    \item $(=, x, f, b)$, where $x$ is a variable, $S(f) = \mathbb{B}$,  and $b \in \mathbb{B}$.
\end{itemize}

Naturally, the domain of $\hfoil_S$ consists of functions from feature names $f \in S$ to values in $S(f) \cup \{\bot\}$, which we call instances of $\hfoil_S$. Continuing with our running example, the function $\es$ such that $\es(\text{age}) = 19.4, \es(\text{parentsTogether}) = 0, \ldots.$ is an instance of $\hfoil_S$.  The semantics for the atomic formulas $\full(x)$, $(\leq, x, f, \tau)$, and $(=, x, f, b)$ is naturally defined as one would expect, by checking whether $\es_x(f) \leq \tau$, and $\es_x(f) = b$, respectively. In order to clarify the semantics of $\pos$, we detail how instances of $\hfoil_S$ are evaluated by a decision tree.

For a decision tree $\T$ compatible with $S$ and an instance $\es$ of $\hfoil_S$, we define $\T(\es)$ inductively:

\begin{itemize}
    \item If $\T$ is a leaf labeled with $\true$, then $\T(\es) = 1$, and $\T(\es) = 0$ if the label is $\false$.
    \item If $\T$ has a root labeled with $(f, \tau)$, left sub-tree $\T_0$, and right sub-tree $\T_1$, then $\T(\es)$ is defined as follows. If $\es(f) \leq \tau$, then $\T(\es) = \T_1(\es)$, and otherwise $\T(\es) = \T_0(\es)$.
     \item If $\T$ has a root labeled with $f$, left sub-tree $\T_0$, and right sub-tree $\T_1$, then $\T(\es)$ is defined as $\T_{f(\es)}(\es)$.
\end{itemize}

It is important to stress that Boolean decision trees, in order to prevent inconsistencies, cannot repeat node labels in any path from the root a leaf. For real-valued decision trees we need a stronger requirement to avoid inconsistencies: if $\T$ has a root labeled with $(f, \tau)$, left sub-tree $\T_0$, and right sub-tree $\T_1$, then all nodes labeled with $(f, \tau_0)$ in $\T_0$ must hold $\tau_0 > \tau$, and similarly all nodes labeled with $(f, \tau_1)$ in $\T_1$ must hold $\tau_1 < \tau$. 

We are now ready to define a binarization procedure.

\newcommand{\B}{\mathcal{B}}
\begin{definition}
A binarization procedure $\B$ is an algorithm that takes: (i) a schema $S$; (ii) an existential formula $\varphi$ in $\hfoil_S$ ; (iii) a decision tree $\T$ compatible with $S$, and returns a formula $\psi \in (\efoil \cup \full)$ together with a binary model $\M$, such that
\[
\T \models \varphi \iff \M \models \psi.
\]
\end{definition}

The rest of this section is dedicated to show an efficient binarization procedure $\B$. First, let us show the intuition behind the procedure with a simple example. Figure~\ref{fig:realdt} depicts a real-valued decision tree $\T$ over the schema $S$ of our running example. Note that, although an instance $\es$ can have any real value as $\es(\text{age})$, $\T$ only distinguishes 4 intervals:
\[
(-\infty, 16], \; (16, 21], \; (21, 25], \; (25, \infty)
\]

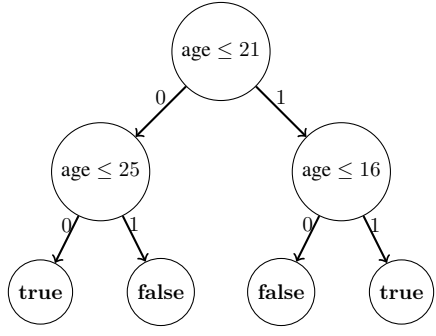
\begin{figure}
    \centering
    \begin{tikzpicture}[scale=0.8, every node/.style={scale=0.8}]
    \node[draw, circle] (r) at (0,0) {$\text{age} \leq 21$} ;
    \node[draw, circle] (l1)at (-2, -2) {$\text{age} \leq 25$} ;
    
    \node[draw, circle] (l2) at (2, -2) {$\text{age} \leq 16$} ;
    
     \node[draw, circle] (l3) at (3, -4) {$\true$} ;
     
     \node[draw, circle] (l4) at (1, -4) {$\false$} ;
     
      \node[draw, circle] (l5) at (-3, -4) {$\true$} ;
     
     \node[draw, circle] (l6) at (-1, -4) {$\false$} ;
     
      \path[->,draw,thick]
    (r) edge node[above] {$0$} (l1)
    (r) edge node[above] {$1$} (l2)
    (l2) edge node[above] {$1$} (l3)
    (l2) edge node[above] {$0$}(l4)
    (l1) edge node[above] {$0$} (l5)
    (l1) edge node[above] {$1$} (l6)
    ;
    \end{tikzpicture}
    \caption{Example of real-valued decision tree. For the sake of clarity, labels are written as $f \leq \tau$ instead of $(f, \tau)$.}
    \label{fig:realdt}
\end{figure}

Now consider the $\hfoil_S$ query:
\[
\varphi = \exists x \, \pos(x) \land (\leq, x, \text{age}, 27)
\]

Consider now instances $\es_1, \es_2$ such that $\es_1(\text{age}) = 26$ and $\es_2(\text{age}) = 28$.
While $\T$ accepts both $\es_1$ and $\es_2$ by traversing the same path, only $\es_1$ is a witness for $\varphi$. This implies that we require a finer partition of the real line into intervals. Namely,

\[
\mathcal{I}_{\text{age}} = (-\infty, 16], \; (16, 21], \; (21, 25], \; (25, 27], \; (27, \infty)
\]

is a \emph{correct partition} for the tuple $(\T, \varphi, \text{age})$. Based on this, as $|\mathcal{I}_{\text{age}}|= 5$, we will use $5-1 = 4$ binary features to encode the age of instances. In particular, the leftmost $1$ among those $4$ binary features will indicate the interval to which the age of an instance belongs, interpreting that if there is no $1$ among the $4$ binary features, it belongs to the last interval. This will then allow to do the following compilation from $\hfoil_S$ to $\foil$:
\[
(\leq, x, \text{age}, 27) \rightsquigarrow \neg \left( x \subseteq \begin{pmatrix} 0, \, 0, \, 0, \, 0, \, \bot, \, \bot, \, \ldots
\end{pmatrix}\right)
\]

As a $\foil$ instance $\es$ not having $0$ in the first four Boolean features that encode age must have at least a $1$ in  one of those Boolean features, and thus, it corresponds to a $\hfoil_S$ instance whose age lies in one of the first four intervals of $\mathcal{I}_\text{age}$, and therefore have $\text{age} \leq 27$.

More formally, assume a real-valued decision tree $\T$ and a $\hfoil_S$ query $\varphi$. Then, for every feature name $f \in S$ such that $S(f) = \mathbb{R}$, we define its \emph{partition set} as follows:
\[
P_f = \{ \tau \mid (f, \tau) \text{ labels a node in } \T, \text{ or } (\leq, x, f, \tau) \text{ appears in } \varphi \text{ for some variable } x  \}
\]

Feature $f$ will be encoded using $|P_f|$ Boolean features. The resulting dimension of instances of $\hfoil_S$ compiled into $\foil$ instances will therefore be $
 n = \sum_{f \in S} |P_f|
$, taking the convention that if $S(f) = \mathbb{B}$ then $|P_f| = 1$.

As $S$ is unordered, but instances of $\foil$ have an ordering of their features, we choose an arbitrary ordering of features names $f_1, \ldots, f_k$, and we associate to them ranges of Boolean features as follows. To $f_1$ we associate the components in the range $[1, |P_{f_1}|]$, and then for $f_i, i > 1$, we associate $[\ell, \ell + |P_{f_i}| - 1]$, where $\ell$ is end of the range associated to $f_{i-1}$ plus $1$.

Therefore comparison of the form $(\leq, x, f, \tau)$, for some variable $x$, will thus be compiled as:
\[
  (\leq, x, f, \tau) \rightsquigarrow   \neg \left( x \subseteq \es_\tau \right)
\]
where, if $\tau$ is the $i$-th smallest element in $P_f$, then $\es_\tau$ is an instance having $0$ in the first $i$ Boolean features associated to $f$, and $\bot$ in the rest. 

We now need to binarize the decision tree $\T$ accordingly. We do so by transforming $\T$ into a BDD $\B(\T, \varphi)$ (not necessarily free) that is \emph{almost}-free, in a precise sense that we will detail. If $\T$ was using a single node $u$ to test whether the age of an instance was at most $27$, we now require several nodes to test for the different Boolean features encoding the age of said instance. In particular, if \text{age} is encoded with Boolean features of indices $[i, \ldots, j]$, and $27$ is the $k$-th smallest value in $P_{\text{age}}$, then $\B(\T, \varphi)$ needs to test that there is a $1$ in among the features of indices $[i, \ldots, i+k-1]$. Thus, we create in $\B(\T)$ a \emph{gadget} for node $u$ of $\T$ that tests whether
\[
\es[i] \lor \es[i+1] \lor \cdots \lor \es[i+k-1].
\]

Correctness is clear from the construction of the binarization procedure. Figure~\ref{fig:afbdd} illustrates $\B(\T, \varphi)$ continuing with the previous example.

\begin{figure}
    \centering
    \begin{tikzpicture}[scale=0.8, every node/.style={scale=0.8}]
    \node[draw, circle] (r) at (0,0) {$1$} ;
    \node[draw, circle] (l1)at (-1, -1) {$2$} ;
    
    \node[draw, circle] (l2)at (-2, -2) {$3$} ;
    
    \node[draw, circle] (l3) at (-3, -3) {$3$} ;
    
     \node[draw, circle] (l4) at (-4, -4) {$4$} ;
    
    
     \node[draw, circle] (lt) at (3, -4) {$\true$} ;
     
     \node[draw, circle] (lf) at (0, -5) {$\false$} ;
     
     
     
      \path[->,draw,thick]
    (r) edge node[above] {$1$} (lt)
     (l1) edge node[above] {$1$} (lt)
      (l2) edge node[above] {$1$} (lt)
       (l3) edge node[above] {$1$} (lt)
        (l4) edge node[above] {$1$} (lt)
    (r) edge node[above] {$0$} (l1)
    (l1) edge node[above] {$0$} (l2)
    (l2) edge node[above] {$0$} (l3)
    (l3) edge node[above] {$0$} (l4)
    (l4) edge node[above] {$0$} (lf)
    ;
    \end{tikzpicture}
    \caption{Illustration of $\B(\T, \varphi)$, assuming $\text{age}$ is the only feature, and considering a query $\varphi = \exists x \pos(x)$ that does not mention any threshold $\tau$.}
    \label{fig:afbdd}
\end{figure}
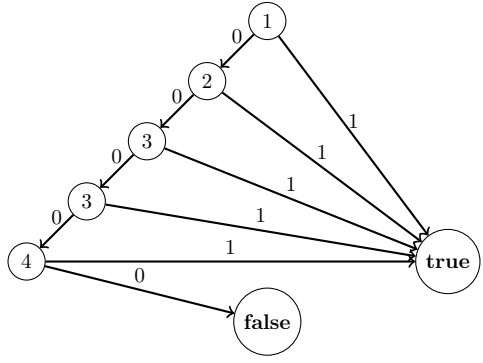

While $\B(\T, \varphi)$ is not necessarily an FBDD, we claim that it is close enough to one in what concerns the evaluation of $\efoil$ formulas. As the proof of Proposition~\ref{prop:dt-obdd-efoil} implies, the only characteristic of models we require for efficient evaluation in the existential fragment is that one can find in polynomial time a determinization of an undetermined instance that is positive for said models. This is clearly not possible for general BDDs, as even checking if there is one positive instance for a BDD is $\np$-hard~\cite{wegener2004bdds}. Surprisingly, the same algorithm for FBDDs presented in the proof of Proposition~\ref{prop:dt-obdd-efoil} turns out to work for $\B(\T, \varphi)$.

In order to see illustrate why this is true, we consider a more sophisticated example of a real-valued decision tree $\T$, presented in Figure~\ref{fig:realdtcomplex}. Note that both the gadgets for nodes labeled $(\text{age}, 21)$ and $(\text{age}, 25)$ will use the Boolean features associated to age, and thus the \emph{freeness} property will be broken. The reason the presented algorithm does not work in general BDDs is that a positive leaf could only be reachable through inconsistent paths, i.e., paths that contain both edges representing that a feature has value $0$ and $1$. In the case of $\B(\T, \varphi)$, features may appear multiple times in a path from root to leaf, but always as part of different gadgets associated to different nodes in $\T$. This implies that, even if an inconsistent choice is made for a particular Boolean feature of $\B(\T, \varphi)$, an inconsistent path in $\B(\T, \varphi)$ still corresponds as a consistent path in $\T$, because paths in $\B(\T, \varphi)$ translate back to paths in $\T$ by considering if gadgets where exited by failing or succeeding the disjunction they represent.

\begin{figure}
    \centering
    \begin{tikzpicture}[scale=0.8, every node/.style={scale=0.8}]
    \node[draw, circle] (r) at (0,0) {$\text{age} \leq 21$} ;
    \node[draw, circle] (l1)at (-2, -2) {$\text{male}$} ;
    
    \node[draw, circle] (lage) at (-3, -4) {$\text{age} \leq 25$} ;
    
    \node[draw, circle] (l2) at (2, -2) {$\text{age} \leq 16$} ;
    
     \node[draw, circle] (l3) at (3, -4) {$\true$} ;
     
     \node[draw, circle] (l4) at (1, -4) {$\false$} ;
     
      \node[draw, circle] (l5) at (-4, -6) {$\true$} ;
     
     \node[draw, circle] (l6) at (-1, -4) {$\false$} ;
     
      \node[draw, circle] (l7) at (-2, -6) {$\false$} ;
     
      \path[->,draw,thick]
    (r) edge node[above] {$0$} (l1)
    (r) edge node[above] {$1$} (l2)
    (l2) edge node[above] {$1$} (l3)
    (l2) edge node[above] {$0$}(l4)
    (l1) edge node[above] {$0$} (lage)
    (l1) edge node[above] {$1$} (l6)
    (lage) edge node[above] {$0$} (l5)
    (lage) edge node[above] {$1$} (l7)
    ;
    \end{tikzpicture}
    \caption{Example of real-valued decision tree. For the sake of clarity, labels are written as $f \leq \tau$ instead of $(f, \tau)$.}
    \label{fig:realdtcomplex}
\end{figure}




\end{document}